\documentclass[twoside,11pt]{article}
\usepackage{jmlr2e}
\usepackage{blindtext}

\usepackage{savesym,multirow,url,xspace,graphicx,hyperref,enumitem,color}
\usepackage[titletoc]{appendix}
\usepackage{subcaption}
\usepackage{dsfont}
\usepackage{multicol}
\usepackage{blkarray}
\usepackage{wrapfig}

\usepackage{etoolbox}
\usepackage[utf8]{inputenc} 
\usepackage[T1]{fontenc}    
\usepackage{hyperref}       
\usepackage{url}            
\usepackage{booktabs}       
\usepackage{amsfonts}       
\usepackage{nicefrac}       
\usepackage{microtype}      
\usepackage{tikz}
\usepackage{pgfplots}
\usepackage{verbatim}
\usepackage{amsmath,bm}
\usepackage{multirow}
\usepackage{mathtools}
\usepackage{comment}

\usepackage[lined, boxed, ruled, commentsnumbered, noend]{algorithm2e}

\newtoggle{showcomments}
\settoggle{showcomments}{true}

\iftoggle{showcomments}{
\newcommand{\akash}[1]{{\textcolor{purple}{[{\bf A:} #1]}}}
\newcommand{\parthe}[1]{{\textcolor{red}{[{\bf P:} #1]}}}

\newcommand{\annotate}[1]{{\textcolor{blue}{[{\bf Note:} #1]}}}
}
{
\newcommand{\akash}[1]{}
\newcommand{\parthe}[1]{}
\newcommand{\annotate}[1]{}
}

\makeatletter
\newif\if@restonecol
\makeatother

\usepackage{amsmath,amssymb,xfrac}
\setitemize{noitemsep,topsep=1pt,parsep=1pt,partopsep=1pt, leftmargin=12pt}

\newtheorem{assumption}{Assumption}

\makeatletter

\newcommand{\Rmnum}[1]{\expandafter\@slowromancap\romannumeral #1@}
\makeatother
\usepackage{caption}

\newcommand{\defref}[1]{Definition~\ref{#1}}

\newcommand{\figref}[1]{Fig.~\ref{#1}}
\newcommand{\eqnref}[1]{\text{Eq.}~(\ref{#1})}
\newcommand{\secref}[1]{\textnormal{Section}~\ref{#1}}
\newcommand{\appref}[1]{Appendix \ref{#1}}

\newcommand{\thmref}[1]{Theorem~\ref{#1}}
\newcommand{\corref}[1]{Corollary~\ref{#1}}
\newcommand{\propref}[1]{Proposition~\ref{#1}}
\newcommand{\lemref}[1]{Lemma~\ref{#1}}

\newcommand{\assref}[1]{Assumption~\ref{#1}}

\newcommand{\algoref}[1]{Algorithm~\ref{#1}}
\newcommand{\lineref}[1]{Line~\ref{#1}}

\newcommand{\paren} [1] {\ensuremath{ \left( {#1} \right) }}

\newcommand{\bigparen} [1] {\ensuremath{ \Big( {#1} \Big) }}

\newcommand{\bracket}[1]{\left[#1\right]}

\newcommand{\curlybracket}[1]{\ensuremath{\left\{#1\right\}}}
\newcommand{\dist}[2]{\ensuremath{\left\langle#1,\:#2\right\rangle_{\cB}}}

\newcommand{\condcurlybracket}[2]{\ensuremath{\left\{#1\left\lvert\:#2\right.\right\}}}

\def \argmin {\mathop{\rm arg\,min}}

\newcommand{\nats}{\ensuremath{\mathbb{N}}}
\newcommand{\reals}{\ensuremath{\mathbb{R}}}

\newcommand{\sgn}[1]{\operatorname{sgn}\paren{#1}}

\newcommand{\cS}{{\mathcal{S}}}

\newcommand{\cD}{{\mathcal{D}}}

\newcommand{\cF}{{\mathcal{F}}}

\newcommand{\cB}{{\mathcal{B}}}

\newcommand{\cX}{{\mathcal{X}}}

\newcommand{\cC}{{\mathcal{C}}}

\newcommand{\cP}{{\mathcal{P}}}

\newcommand{\cY}{{\mathcal{Y}}}

\newcommand{\bc}{{\mathbf{c}}}

\newcommand{\bx}{{\mathbf{x}}}
\newcommand{\by}{{\mathbf{y}}}

\renewcommand{\tt}[1]{\textit{#1}}
\renewcommand{\sf}[1]{\textsf{#1}}
\newcommand{\bg}[1]{g_{#1}}

\def\BState{\State\hskip-\ALG@thistlm}

\newcommand{\MDA}{\textsf{MDA}}

\usepackage[margin=1in]{geometry}
\usepackage{xcolor} 
\usepackage{empheq}
\usepackage{framed}
\definecolor{shadecolor}{gray}{0.9}

\newcommand{\abs}[1]{\left|#1\right|}
\newcommand{\round}[1]{\left(#1\right)}
\newcommand{\inner}[1]{\left<#1\right>}
\newcommand{\norm}[1]{\left\|#1\right\|}

\def\x{\bm{x}}

\usepackage{booktabs}
\usepackage{natbib}
\hypersetup{hidelinks,colorlinks,citecolor=blue}

\def\mc{\mathcal}
\def\mf{\mathfrak}
\def\msf{\mathsf}

\usepackage{lastpage}

\ShortHeadings{}{Mirror Descent on RKBS}
\firstpageno{1}

\begin{document}

\title{Mirror Descent on Reproducing Kernel Banach Spaces}

\author{\name Akash Kumar \email akk002@ucsd.edu \\
       \addr Department of Computer Science and Engineering\\
       University of California San Diego\vspace{-2mm}
       \AND
       \name Mikhail Belkin \email mbelkin@ucsd.edu \\
       \addr Department of Computer Science and Engineering\\
       Halıcıoğlu Data Science Insititute\\
       University of California San Diego\vspace{-2mm}
       \AND
       \name Parthe Pandit \email pandit@iitb.ac.in \\
       \addr Center for Machine Intelligence and Data Science (C-MInDS)\\
Indian Institute of Technology Bombay, India\vspace{-5mm}}


\maketitle

\begin{abstract}
Recent advances in machine learning have led to increased interest in reproducing kernel Banach spaces (RKBS) as a more general framework that extends beyond reproducing kernel Hilbert spaces (RKHS). These works have resulted in the formulation of representer theorems under several regularized learning schemes. However, little is known about an optimization method that encompasses these results in this setting. This paper addresses a learning problem on Banach spaces endowed with a reproducing kernel, focusing on efficient optimization within RKBS. To tackle this challenge, we propose an algorithm based on mirror descent (MDA). Our approach involves an iterative method that employs gradient steps in the dual space of the Banach space using the reproducing kernel.

We analyze the convergence properties of our algorithm under various assumptions and establish two types of results: first, we identify conditions under which a linear convergence rate is achievable, akin to optimization in the Euclidean setting, and provide a proof of the linear rate; second, we demonstrate a standard convergence rate in a constrained setting. Moreover, to instantiate this algorithm in practice, we introduce a novel family of RKBSs with \(p\)-norm (\(p \neq 2\)), characterized by both an explicit dual map and a kernel.

\end{abstract}


\section{Introduction}

In supervised machine learning, we are given a set of observations $\cD_n = \curlybracket{(\x_i,y_i)}_{i=1}^n$, where the inputs $\x_i$, $i = 1, 2,\ldots,n$ are sampled from a data space $\cX$ with corresponding outputs $y_i$ from a label set $\cY$. The task is to find a function $\hat{f}: \cX \to \cY$, chosen from a \textit{a priori} fixed model class $\cF$, that best predicts $\hat{f}(\x_i) \approx y_i$, $i = 1,2,\ldots$. The choice of the model class $\cF$ is usually based on a prior belief in the expressivity of predictive functions that could lead to the optimal classifier $f^*$ (which may not be in $\cF$), as dictated by the nature/environment.

Typically, to find $\hat{f}$, we consider a minimization problem over the observations $D_n$ with respect to a loss function $\ell: \cY \times \cY \to \reals$ as follows:
\begin{align}
    \argmin_{f \in \cF} \frac{1}{n} \sum_{i=1}^n \ell(f(\x_i),y_i) \label{eq: optERM}
\end{align}

Various optimization techniques have been proposed to minimize the optimization error in approximating a solution to \eqnref{eq: optERM}, with first-order gradient methods~\citep{nemirovski1983problem,yuri,Lee2016GradientDO} being the most well-known in the parametric setting, e.g., neural networks.

The choice of $\cF$ plays a pivotal role in minimizing the optimization error while approximating the optimal classifier $f^*$ ~\citep{tradeoffML}. A fundamental question arises: \textit{Can a function space $\cF$ be constructed to minimize approximation error without trading off with optimization error?}

Traditionally, kernel methods, particularly in reproducing kernel Hilbert spaces (RKHSs), have been proposed to bound and lower the approximation error in a regularized empirical setting 
\begin{gather}
    \min_{f \in \cF} \sum_{i=1}^n \ell(y_i, f(x_i)) + \lambda\cdot \Psi(f),
    \textnormal{where } \Psi: \cF \to \reals_{\ge 0}, \lambda > 0, \label{eq: regular}
\end{gather}
However, while RKHSs exhibit function space optimality~\citep{Scholkopf2001LearningWK}, where the optimal solution to \eqnref{eq: regular} has a representer that facilitates efficient optimization methods, they may lack expressive power in terms of approximation error. 

To overcome this limitation and enrich the diversity of geometric structures and norms, recent research has explored alternative model classes, such as non-Hilbertian \textit{Banach spaces} (characterized by norms that don't adhere to the parallelogram law), aiming for improved approximation capabilities. One noteworthy framework in this context is the \textit{reproducing kernel Banach spaces} (RKBS), as introduced in \cite{zhang09b,Lin2022}. Within the RKBS framework, several significant problems—such as minimal norm interpolation, regularization networks, support vector machines, sparse learning, and multi-task learning—have been formulated and investigated~\citep{song201396,zhang09b,Ye13,Zhang2012RegularizedLI,xu2019generalized,Xu2023SparseML,wang2023sparse}.

Additionally, significant efforts have been made to characterize the function spaces learned by neural networks through the perspective of Banach spaces~\citep{Bach2014BreakingTC,Parhi2019TheRO,Ongie2020A,Parhi2020BanachSR,Wright2021TransformersAD}. More recently, there has been a shift towards treating this characterization as an optimization problem over RKBS~\citep{Spek2022DualityFN,BARTOLUCCI2023194,shilton23a, Parhi2023FunctionSpaceOO}.

Although the aforementioned findings have significantly contributed to our understanding of the approximation capabilities of Banach spaces, the extent to which these formulations address the corresponding \textit{optimization error}, specifically in terms of statistically efficient and/or provable methods for identifying solutions to \eqnref{eq: regular}, remains unexplored in the setting of reproducing kernel Banach spaces.

In this work, we address this gap by proposing an algorithm based on \textit{mirror descent}~\citep{nemirovski1983problem}. We focus on a general minimization problem defined over the domain of reproducing kernel Banach spaces and demonstrate that the reproducing property allows us to prescribe a gradient update in the dual of the Banach space. Our contributions aim to bridge the gap between approximation and optimization errors in the context of function space design, providing a novel perspective and practical algorithmic results for RKBS.

\begin{figure}[t!]
    \centering
    \begin{minipage}{0.45\textwidth}
        \centering
        \begin{tikzpicture}[scale=1.8, every node/.style={scale=1.2}]

    \tikzstyle{every node}=[font=\small]
    \tikzstyle{point}=[circle, fill=black, inner sep=0pt, minimum size=4pt]

    \draw [thick] (-1,0) ellipse (0.6 and 1);
    \draw [thick] (1,0) ellipse (0.6 and 1);

    \node [point,label=above:$f_t$] (x) at (-1,0.3) {};
    \node [point,label=above:$f_{t-1}$] (y) at (-1,-0.5) {};
    \node [point,label=right:$g_{t}$] (xprime) at (1,0.5) {};
    \node [point,label=right:$g_{t-1}$] (yprime) at (1,-0.5) {};

    \node [below] at (-1,-1) {$\mathcal{B}$};
    \node [below] at (1,-1) {$\mathcal{B}^*$};
    \node [below] at (-.03,-0.68) {$\partial \Phi$};
    \node [below] at (-.03,.92) {$\partial \Phi^{-1}$};
    \node [below] at (1.45,0.1) {\textnormal{gradient step}};
    \draw [<-, thick,color=red, dashed] (x) .. controls (-0.5,0.7) and (0.5,0.7) .. (xprime);
    \draw [->, thick,color=red, dashed] (y) .. controls (-0.5,-0.7) and (0.5,-0.7) .. (yprime);
    \draw [->, thick,color=gray, dash dot] (yprime) .. controls (0.8,0) and (0.8,0.2) .. (xprime);

\end{tikzpicture}

        \caption{Functional MDA}
    \end{minipage}
    \hfill
    \begin{minipage}{0.45\textwidth}
        \centering
        \begin{tikzpicture}[scale=1.8, every node/.style={scale=1.2}]

    \tikzstyle{every node}=[font=\small]
    \tikzstyle{point}=[circle, fill=black, inner sep=0pt, minimum size=4pt]

    \draw [thick] (-1,0) ellipse (0.6 and 1);
    \draw [thick] (1,0) ellipse (0.6 and 1);

    \node [point,label=above:$f_t$] (x) at (-1,0.3) {};
    \node [point,label=above:$f_{t-1}$] (y) at (-1,-0.5) {};
    \node [point,label=right:$\bc^{(t)}$] (xprime) at (1,0.5) {};
    \node [point,label=right:$\bc^{(t-1)}$] (yprime) at (1,-0.5) {};

    \node [below] at (-1,-1) {$\mathcal{B}$};
    \node [below] at (1,-1) {$\reals^n$};
    \node [below] at (1,-1.2) {(Induced by $\curlybracket{K(\bx_i,\cdot)}$)};
    \node [below] at (-.03,-0.68) {$\partial \Phi$};
    \node [below] at (-.03,.92) {$\partial \Phi^{-1}$};
    \node [below] at (1.45,0.1) {\textnormal{gradient step}};
    \draw [<-, thick,color=red, dashed] (x) .. controls (-0.5,0.7) and (0.5,0.7) .. (xprime);
    \draw [->, thick,color=red, dashed] (y) .. controls (-0.5,-0.7) and (0.5,-0.7) .. (yprime);
    \draw [->, thick,color=gray, dash dot] (yprime) .. controls (0.8,0) and (0.8,0.2) .. (xprime);

\end{tikzpicture}

        \caption{Our MDA for RBKSs}
    \end{minipage}
    \caption{
The schematic diagram for the mirror descent algorithm is presented. The first image represents the general functional form of the mirror descent algorithm. In the second image, we illustrate the update rule for Reproducing Kernel Banach Spaces (RKBS), which requires updates in \(\mathbb{R}^n\) corresponding to the kernel evaluations on the training data points.\vspace{-3mm}}
    \label{fig: md}
\end{figure}
We summarize the contributions of the work below:
\begin{enumerate}
    \item \textbf{Mirror Descent:} We address a general minimization problem given by \eqnref{eq: optERM}, encompassing both regularized and unregularized settings. In this formulation, the function space $\cF$ is a non-Hilbertian Banach space, signifying that the norm, denoted as $||\cdot||_\cF$, does not correspond to an inner product. 
    Unlike Hilbert spaces, Banach spaces exhibit a geometric difference in that their dual spaces need not be naturally identifiable with the underlying space. Thus, the gradient of the loss functional $\ell$ in the first argument does not exist within the function space.
    Consequently, we turn to the dual of the Banach space, where the gradient steps are executed. These updates in the dual space are then reflected back to the primal Banach space, rendering the algorithm as mirror descent over the Banach space.
    
    Given that derivatives and updates are computed over functions, this approach represents a functional form of mirror descent, distinguishing it from many other works. Furthermore, our focus extends to Banach spaces endowed with a reproducing property, defining them as reproducing kernel Banach spaces. This unique characteristic enables the representation of the mirror descent algorithm as kernel evaluations for updates in the dual space. This, in turn, facilitates simpler and more tractable updates in the dual space expressed in terms of a kernel function. 
    
    For example, an instant of the algorithm for optimizing square-error loss functionals with 
    regularization parameter $\lambda > 0$ and using the regularization functional as a mirror map has the form (see \figref{fig: md}):
    \textcolor{darkgray}{
    \begin{align*}
        g_t &\gets (1 - 2\eta \lambda)\cdot g_{t-1} - 2\eta \times (\tt{loss differential}) \\ 
    f_{t} &\gets \sf{InverseMirrormap}(g_t)
    \end{align*}}
    where $f_t$ and $g_t$ are in the primal and dual Banach spaces respectively, and $\eta$ is the learning rate. Now, if the reproducing kernel $K$ is known, then this could be simplified as \textcolor{darkgray}{
    \begin{align*}
    \bc^{(t)} &\gets (1 - 2\eta \lambda)\bc^{(t-1)} - 2\eta(f_{t-1}(\x) - Y)\\ 
    f_{t} &\gets  \sf{InverseMirrormap} (g_t),\,\,g_t = \sum_{i = 1}^n \bc^{(t)}_i \cdot K(\x_i,\cdot)
\end{align*}}
    where $\bc^{(t)} \in \reals^n$ ($n$ depends on the training set size), and $f_t(\x) = (f_t(\x_1), f_t(\x_2),\ldots, f_t(\x_n))$, $Y = (y_1,y_2,\ldots,y_n)$.
    
    We provide the relevant definitions on Banach spaces and the corresponding functional analysis in \secref{sec: prelim}, with the MDA discussed in \secref{subsec: probsetup} in details.

\begin{figure}[t]
    \centering
    \includegraphics[width=.9\linewidth]{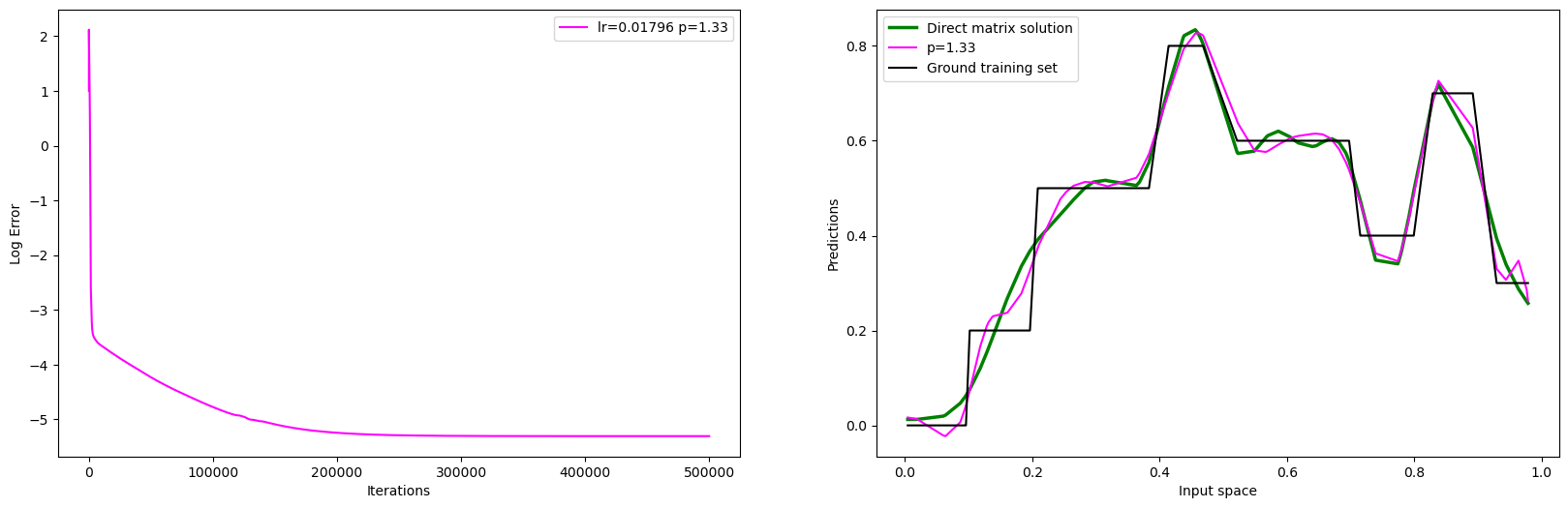}
    \caption{
We employ mirror descent on \( p \)-norm Reproducing Kernel Banach Spaces (RKBS) for a step function defined on the interval \([0,1]\). Our training set consists of 80 points, with 15 centers randomly selected from this set. The first image illustrates the logarithm of the training error over iterations. The second image compares the approximation for \( p = 1.33 \) (depicted by the purple curve) against a solution obtained using NumPy (shown by the green curve). In this comparison, the NumPy solution solves \( K_g \cdot \alpha = Y \), where \( K_g \) is the similarity matrix computed over the (training, centers) pairs using a Gaussian kernel.
}
    \label{fig:example_image}
    \vspace{-2mm}
\end{figure}

    \textbf{Instantiation on an example RKBS:} To implement the Mirror Descent Algorithm (MDA), we construct a finite-center based Reproducing Kernel Banach Space (RKBS) (see \secref{sec: example}). Informally, the Banach space is defined as the set of functions of the following form:
 \begin{align*}
        \cB := \curlybracket{f : f = \sum_{i=1}^n \alpha_i\cdot H(\cdot,\bc_i), \paren{\sum_{i=1}^n\paren{|\alpha_i|^p}}^{\frac{1}{p}} \lneq \infty}
    \end{align*}

    where $\alpha_i \in \mathbb{R}$, $i = 1,2,\ldots, n$, $\mathbf{c}_i$ are samples/centers in a fixed sample space, and the bivariate function $H$ is defined over pairs of points in that sample space. With minimal assumptions on the linear independence of $H(\cdot,\mathbf{c}_i)$, $i =1,2,\ldots$, we can demonstrate that it is a reproducing kernel Banach space, and a unique kernel can be explicitly specified.

Additionally, we introduce \(\ell_p\)-type mirror maps, also referred to as dual maps, which, when combined with the kernel, precisely recover the gradient iterations of the Mirror Descent Algorithm (MDA) within this Banach space. In the context of nonregularized learning, the iterative update steps can be rigorously expressed as follows:
\begin{align*}
    \beta^{(t)} &\gets \beta^{(t-1)} - \eta \cdot \left(\widehat{H}^{\top} \left(\widehat{H}\alpha^{(t-1)} - Y\right)\right), \\
    \alpha^{(t)} &\gets \frac{\sgn{\beta^{(t)}} |\beta^{(t)}|^{q-1}}{\norm{\beta^{(t)}}_q^{q-1}},
\end{align*}
where \(\beta\) and \(\alpha\) represent the dual and primal parameters, respectively, and \(\widehat{H}\) denotes the similarity kernel matrix derived from the training data and centers. This construction of the RKBS is detailed further in \secref{sec: example}, where we also demonstrate its application to a learning task. An illustration is provided in \figref{fig:example_image} in which we show the convergence of the $p$-norm RKBS for learning a step function where $H$ is an asymmetric Lab-RBF kernel~\citep{he2024learninganalysiskernelridgeless}.
    

     \item \textbf{Convergence of MDA:} More generally, the dual space to a Banach space, comprising real-valued linear transformations defined over the Banach space, may exhibit an unfavorable topology, resulting in the lack of a proper mirroring operation in the mirror descent algorithm. To mitigate such situations, we assume the RKBS is \textit{reflexive}, meaning the dual of the dual space is isometrically isomorphic to the Banach space. 
    
    Our main theoretical result establishes the convergence of MDA for reflexive Banach spaces that are Hilbertizable (defined as spaces isomorphic to a Hilbert space\footnote{\tt{not} isometrically isomorphic}). This is stated informally as (formally as \thmref{thm: uncons})
    
    \textbf{Informal Theorem 1}: \tt{If the underlying reflexive Banach space is Hilbertizable, and the loss functional and mirror maps possess properties of (functional notions) $\mu$-strong convexity (see \defref{def: strongconv}) and $\gamma$-smoothness (see \defref{def: smooth}), then in the unconstrained optimization setting (regularized and non-regularized) with a unique minimizer, MDA achieves linear rate depending on the smoothness and strong convexity parameters and a choice of learning rate.}

    We show a negative result on the existence of functionals possessing both $\mu$-strong convexity and $\gamma$-smoothness for Banach spaces \tt{not} isomorphic to a Hilbert space.
    Under the same notion, this rate also extends under the (functional form of) \tt{Polyak-\L{}ojasiewicz inequality}, generally studied as a weaker notion than convexity for the loss functional. 
    
    Furthermore, under mild assumptions, we extend the analysis to demonstrate the convergence of the algorithm in the constrained setting, achieving a rate of $\mathcal{O}\left(\frac{1}{\sqrt{t}}\right)$, where $t$ is the number of iterations. Formal definitions of (functional) smoothness and convexity of a functional (defined on Banach spaces) are detailed in \secref{subsec: funcopt}.
\end{enumerate}


\subsection{Related Work}
\paragraph{Mirror Descent:} The mirror descent method, introduced in \cite{nemirovski1983problem} for convex optimization problems and later analyzed in \cite{beck2003mirror}, is a well-established first-order method known for its effectiveness in designing algorithms for non-Euclidean geometries. It is particularly recognized for its almost dimension-free rate of convergence, as documented in works like \cite{beck2003mirror, BenTal2001TheOS, Censor1992ProximalMA, eckstein, BianchiMD}. More recently,  implicit regularization in deep neural networks via mirror descent has been studied in \cite{mdnavid, Sun2023AUA}.

The key principle of the method involves a strongly convex potential function that induces a Bregman divergence (metric) over the problem space, as detailed in \cite{Bauschke2017ADL,bubeck2015convex,azizan2021stochastic}. Subsequently, a gradient descent step is taken in the dual space, which is considered a space of transformations. Previous research has primarily focused on scenarios where the underlying problem space is a Euclidean vector space. In our work, we extend this framework to non-Hilbertian Banach spaces of functions. In this setting, we propose a mirror descent algorithm tailored for a Banach space with a reproducing property, demonstrating favorable convergence properties under standard assumptions.



\paragraph{RKBS:} Formally, the concept of reproducing kernel Banach spaces (RKBS) was introduced by \cite{zhang09b} and later expanded upon in \cite{Lin2022}. This framework emerged as a systematic approach to studying learning in Banach spaces, serving as an extension of the well-established reproducing kernel Hilbert spaces (RKHS). In contrast to Hilbert spaces, RKBSs exhibit more intricate geometrical forms and diverse norms. Additionally, RKBS accommodates various kernel functions, including asymmetric kernels~\citep{zhang2018categorization} and non-positive definite kernels~\citep{bharath2011}. Recent works have proposed different constructions of RBKSs, such as reflexive RKBS in \citet{Lin2022}, semi-inner product (s.i.p) RKBS (refer to \cite{zhang09b, ZHANG20111}), $\ell^1$ norm RKBS in \citet{song201396}, and a class of $p$-norm RKBSs utilizing generalized Mercer kernels~\citep{xu2019generalized}.

While these works have established Representer theorems for various minimization problems in machine learning (detailed in \cite{Lin2022} and \cite{Wang2023SparseRT}), a notable gap persists: the absence of a known computational algorithm to find an optimal solution. Our work addresses this gap by presenting an algorithm specifically designed for reflexive RKBSs. We focus on scenarios where an explicit functional form of the point evaluation functional in the dual space of the RKBS can be identified, represented by a unique reproducing kernel.


\paragraph{Learning in Banach spaces:} There is a growing interest in understanding and establishing connections between learning problems formulated over Banach spaces~\citep{bennett2000duality, argyriou10a, micchelli2004, micchelli2007, unser16, Parhi2019TheRO, srinivasan2022contracting, Parhi2023FunctionSpaceOO,shilton23a}.

Numerous significant problems, including $p$-norm coefficient-based regularization \citep{SHI2011286, song201396, tong10}, large-margin classification~\citep{der07a, fasshauer2015115, zhang09b}, lasso in statistics~\citep{tibshirani}, function spaces of trained neural networks~\citep{Parhi2021WhatKO, shilton23a}, random Fourier features (RFF) for asymmetric kernels~\citep{He2022RandomFF}, and sparsity in machine learning~\citep{song201396,Xu2023SparseML}, have been investigated within the framework of Banach spaces.


\paragraph{Neural Networks and Banach spaces:} 
Recently, there has been a growing interest in understanding deep neural networks through the framework of learning schemes in Banach spaces~\citep{BARTOLUCCI2023194, venkatesh2023, E2018APE, Spek2022DualityFN, Parhi2023FunctionSpaceOO, chung2023barron}. A pivotal question in this context is to characterize the function spaces that deep networks learn or represent~\citep{Bach2014BreakingTC,Gribonval2019ApproximationSO,Ongie2020A,Parhi2020BanachSR,Savarese2019HowDI}. In a series of works, \citet{Parhi2019TheRO,Parhi2020BanachSR, Parhi2021WhatKO} established a representer theorem connecting deep ReLU networks with data fitting problems over functions from specific Banach spaces. 
This line of study was further extended for neural networks with univariate nonlinearity to reproducing kernel Banach spaces~\citep{BARTOLUCCI2023194,Spek2022DualityFN}, and later to multivariate nonlinearity in \citet{Parhi2023FunctionSpaceOO}. Consequently, any computational algorithm developed for optimization in RKBS would have immediate implications for these results. Additional references on this topic can be found in \citep{Chen2023PrimalAttentionST,chung2023barron,Wright2021TransformersAD}.

\vspace{-5mm}





\subsection{Roadmap}

In Section \ref{sec: prelim}, we introduce the necessary notation, mathematical formulations, and key definitions related to Reproducing Kernel Banach Spaces (RKBS). Section \ref{subsec: probsetup} outlines the problem setup and describes the functional mirror descent algorithm. In Section \ref{sec: mainthm}, we present our main theoretical results, along with the corresponding proofs of the algorithm's convergence properties. Finally, Section \ref{sec: example} provides the construction of the $p$-norm finite-center RKBS and a detailed implementation of the mirror descent algorithm. 

\section{Preliminaries}\label{sec: prelim}

\noindent\tt{Notations}: Let a set of letters $f, g, h$ denote elements of a vector space or a dual space. An element of the input space, generally denoted by $\cX$, is represented as $\x$. Vectors in a Euclidean space are denoted either as $\alpha$ or $\beta$. The learning rate is denoted as $\eta$. Function spaces, such as Hilbert space, Banach space, among others, are denoted as $\mc H,\mc B, \mc C$. Bivariate functions (over fixed input spaces) or functionals over fixed function spaces are denoted as $G, H, K$. Bregman divergence over function spaces is denoted by the notation $\mf D$. The dual of an element/space is denoted by $\_^*$. There are two sets of notations for non-linear functionals: $\msf A, \msf B,\msf C,\msf L$ (e.g., loss functional) and $\Phi, \Psi$ (e.g., mirror maps). All the constants in the theoretical results are denoted by Greek symbols--$\rho, \kappa, \lambda, \gamma$.

\noindent\tt{Theory of Reproducing kernel Banach spaces}: Assume $\cX \subseteq \reals^d$ is a \tt{locally} compact Hausdorff space (unless stated otherwise). We consider a family of real-valued functions $\cB:= \condcurlybracket{f}{f: \cX \to \reals}$ over $\cX$. We assume that $\cB$ is a complete \tt{vector space} endowed with the norm $||\cdot||_{\cB}$, \tt{i.e.} $(\cB,||\cdot||_{\cB})$ forms a \tt{Banach space}. Furthermore, we impose the condition that the norms are non-Hilbertian\footnote{(\tt{i.e.} they don't satisfy the parallelogram law)}. 

For a Banach space $\cB$, we denote its \tt{dual} space as $\cB^*$, \tt{i.e.} a set of of real-valued linear transformations on $\cB$, such that for any $g \in \cB^*$
\begin{align*}
    ||g||_{\cB^*} := \sup_{f \in \cB,\, ||f||_{\cB} \le 1} g(f)
\end{align*}
Thus, by definition, the dual space $(\cB^*, ||\cdot||_{\cB^*})$ is a Banach space. A transformation $g \in \cB^*$ induces a natural action on any $f \in \cB$ which is denoted by the bilinear operator, aka duality bracket, $\inner{\cdot,\cdot}_{\cB}: \cB \times \cB^* \to \reals$ such that
\begin{align*}
    \inner{f, g}_{\cB} = g(f)
\end{align*}
Since $g$ is a linear transformation the operation is bilinear in two arguments of the bracket. \noindent The \textbf{bidual} space of a Banach space $\cB$ is the dual of the dual space and is denoted by 
$$\cB^{**} := (\cB^*)^*$$
There is a natural map $\iota = \iota_{\cB} : \cB \to \cB^{**}$ which assigns to every element $f \in \cB$ the linear functional $\iota(f) : \cB^* \to \reals$ whose value on any $\bg{} \in \cB^*$ is obtained by evaluating the linear bounded functional $\bg{}: \cB \to \reals$ at $f$, \tt{i.e.} $\bg{}(f)$.
\begin{align}
  \iota(f)(\bg{}) = \bg{}(f), \quad \forall f \in \cB,\,\, \forall \bg{} \in \cB^* \label{eq: isometric}  
\end{align}
As a consequence of Hahn-Banach theorem, the linear map $\iota$ is an isometric embedding. Conversely, an interesting class of Banach spaces are ones where $\cB^{**}$ can be identified as $\cB$, more formally, they are called \tt{reflexive} spaces as studied in \cite{soloman}, also defined below
\begin{definition}[Reflexive Banach spaces] A real normed vector space
$\cB$ is called reflexive if the isometric embedding $\iota : \cB \to \cB^{**}$ in \eqnref{eq: isometric} is bijective. 
\end{definition}
In other words, a reflexive Banach space $\cB$ has the property that $\cB \simeq \cB^{**}$. In this section, we would use symbol $\cB$ for the bidual space as if they \tt{identify} the same.

In order to design a computational algorithm, we are interested in understanding certain properties of a Banach space, in particular, its \tt{reproducing} property, aka a kernel. We adopt the definitions and treatment for a reproducing kernel Banach space (RKBS) as detailed in \cite{Lin2022}. We provide formal definitions in the following before stating some useful connections.
\begin{definition}[Reproducing kernel Banach spaces (RKBS)]\label{def: RKBS} A reproducing kernel Banach space $\cB$ on a prescribed nonempty set $\cX$ is a Banach space of functions on $\cX$ such that for every ${\x} \in \cX$, the point evaluation functional~\footnote{$\delta_{\x}(f) = f(\x)$ for any $f \in \cB$ and $\x \in \cX$} $\delta_{{\x}}\in\cB^*$ on $\cB$ is continuous, i.e., there exists a positive constant $C_{{\x}} \geq 0$ such that
\begin{align*}
    |\delta_{{\x}}(f)| = |f({\x})| \le C_{{\x}}\cdot||f||_{\cB},\quad \forall f \in \cB
\end{align*}
    
\end{definition}

\begin{definition}[Reproducing kernel]\label{def: RK}
Assume \( \cB_1 \) is an RKBS on a set \( \Omega_1 \). If there exists a Banach space \( \cB_2 \) of functions on another set \( \Omega_2 \), a continuous bilinear form \( \langle \cdot, \cdot \rangle_{\cB_1 \times \cB_2} \), and a function \( K \) on \( \Omega_1 \times \Omega_2 \) such that \( K({\x}, \cdot) \in \cB_2 \) for all \( {\x} \in \Omega_1 \) and \( K(\cdot, \by) \in \cB_1 \) for all \( \by \in \Omega_2 \) and
\begin{equation}
f({\x}) = \langle f, K({\x}, \cdot) \rangle_{\cB_1 \times \cB_2} \text{ for all } {\x} \in \Omega_1 \text{ and for all } f \in \cB_1, \label{eq: RK1}
\end{equation}
then we call \( K \) a reproducing kernel for \( \cB_1 \). If, in addition, \( \cB_2 \) is also an RKBS on \( \Omega_2 \) and it holds \( K(\cdot, \by) \) for all \( \by \in \Omega_2 \) and
\begin{equation}
g(\by) = \langle K(\cdot, \by), g \rangle_{\cB_1 \times \cB_2} \text{ for all } \by \in \Omega_2 \text{ and all } g \in \cB_2, \label{eq: RK2}
\end{equation}
then we call \( \cB_2 \) an adjoint RKBS of \( \cB_1 \) and call \( \cB_1 \) and \( \cB_2 \) a pair of RKBSs. In this case, \( \widetilde{K}({\x}, \by) \coloneqq K(\by, {\x}) \) for ${\x} \in \Omega_2$, $\by \in \Omega_1$, is a reproducing kernel for \( \cB_2 \).    
\end{definition}
\eqnref{eq: RK1} and \eqnref{eq: RK2} are called the \tt{reproducing properties} for the kernel \( K \) in RKBSs \( \cB_1 \) and \( \cB_2 \), respectively. Note that for different choices of the bilinear form \(\langle \cdot, \cdot \rangle_{\cB_1 \times \cB_2} \) lead to a potentially different reproducing kernel. In this work, we are interested in the case where $\Omega_1 = \cX$ $, \Omega_2 = \cB$ and $ \cB_2 = \cB^*$. In this regard, if we use the \tt{canonical} bilinear form, \tt{i.e.} the duality bracket, then we obtain a \tt{unique} reproducing kernel for the RBKS $\cB$. In \secref{sec: example}, we study this unique kernel corresponding to an RKBS.

\subsection{Basic Definitions for Functional Optimization}\label{subsec: funcopt}

Here, we consider standard notions and definitions for calculus over functionals. It requires a careful treatment to extend the inherent definitions for the Euclidean vectors spaces. We assume that our function space is a Banach space $\paren{\cB(\cX), \norm{\cdot}_\cB}$ of real valued functions on $\cX$. Let $\cD \subseteq \cB$ be an open set and $\msf F$ a real-valued functional on this domain. First, we provide some necessary definitions for optimization over the function space by defining the notions of differentiability. Then, we talk about the notions of convexity, smoothness, Lipsitzness for $\msf F$; extending to the Polyak-\L{}ojasiewicz inequality. 

\begin{definition}[G\^{a}teaux differential] Let $\msf F: \cD \to \reals$ be a nonlinear transformation. Let $f \in \cD \subseteq \cB$ and $h$ be arbitrary in $\cB$. If the limit 
\begin{align}
    \partial_f \msf F (h) := \lim_{\gamma \to 0} \frac{1}{\gamma}\bracket{\msf F(f + \gamma\cdot h) - \msf F(f)} \label{eq: gatone}
\end{align}
exists, it is called the G\^{a}teaux differential of $\msf F$ at $f$ with the increment $h$. If \eqnref{eq: gatone} exists for each $h \in \cB$, then $\msf F$ is called G\^{a}teaux differentiable at $f$.
\end{definition}
For the scope of this work, we have stated the definition of G\^{a}teaux differential in the setting of a Banach space. But it could be easily extended to any vector space that need not have a norm.

\begin{definition}[Fr\'{e}chet derivative] Consider the transformation $\msf F$ as defined above. If for fixed $f \in \cD$ and any $h \in \cB$ there exists $\partial_f \msf F(h) \in \reals$ which is linear and continuous with respect to h such that
\begin{align}
    \lim_{\norm{h}_{\cB} \to 0} \frac{\msf F(f + h) - \msf F(h) - \partial_f \msf F(h)}{\norm{h}_\cB} = 0,
\end{align}
then $\msf F$ is said to be Fr\'{e}chet differentiable at $f$ and $\partial_f \msf F (h)$ is said to be the Fr\'{e}chet differential of $\msf F$ at $f$ with increment $h$. 
\end{definition}

As is common in the literature of real analysis, we state analogous definitions on the convexity and smoothness of a loss functional. 

\begin{definition}[Convexity] We say a functional $\msf F: \cD \to \reals$ is convex if for any $f, f' \in \cB$ and $\lambda \in \bracket{0,1}$ we have
\begin{align*}
    \msf F (\lambda f + (1 + \lambda) f') \le \lambda \msf F(f) + (1 + \lambda) \msf F(f')
\end{align*}
If the inequality is strict we call $\msf F$ a strictly convex functional. 
    
\end{definition}
In the functional analysis literature (see \cite{Zlinescu1983OnUC}), a slightly stronger notion of convexity is uniform convexity as defined below. \begin{definition}[Uniform convexity]
We say a functional $\msf F: \cD \to \reals$ is \textit{uniformly convex} if there exists $\rho : \mathbb{R}_+ \rightarrow \overline{\mathbb{R}}_+$ (with $\rho(t) = 0 \Leftrightarrow t = 0$)  such that
\[
\msf F(\lambda f + (1 - \lambda)f') \leq \lambda \msf F(f) + (1 - \lambda) \msf F(f') - \lambda(1 - \lambda)\rho(\|f - f'\|)
\]
for all $f, f' \in \text{dom}\,\, \msf F$ and all $\lambda \in \bracket{0, 1}$.
\end{definition}
Existence of such functionals strictly implies that the Banach space $\cB$ is reflexive (see Theorem 3.5.13 in~\cite{Zlinescu1983OnUC}). We discuss the requirement of reflexivity of the Banach space in \secref{sec: mainthm}.

With this, we define the notion of subgradients which is used to define the strong notion of convexity and smoothness of a functional.


\begin{definition}[Subgradients]
    For a functional $\msf F: \cB \to \reals$, we define its set of subgradients at any $f_0 \in \cB$ as
    \begin{align*}
        \partial_{f_0} \msf F := \curlybracket{ g \in \cB^*; \forall f \in \cB, \msf F(f) \ge \msf F(f_0) + \inner{f - f_0, g}_{\cB}}
    \end{align*}
\end{definition}
Here, we haven't stated any condition on the Fr\'{e}chet Derivative of $\msf F$, which if exists renders a singleton set for subgradients.

\begin{definition}[$\mu$-strongly convex]\label{def: strongconv} We say a functional $\msf F: \cD \to \reals$ is $\mu$-strongly convex if for all $f_0 \in \cB$, there exists $g \in \partial_{f_0}\msf F$ such that 
\begin{align*}
    \forall f \in \cB,\hspace{5mm} \msf F(f) - \msf F(f_0)  \geq \dist{f-f_0}{g} + \frac{\mu}{2} ||f - f_0||^2_{\cB} 
\end{align*}
\end{definition}
Note that $\mu$-strong convexity implies that the functional $\msf F$ is uniformly convex for $\rho(t) = \frac{\mu}{2} t^2$.

\begin{definition}[$\gamma$-smoothness]\label{def: smooth} We say a functional $\msf F: \cD \to \reals$ is $\gamma$-smooth if for all $f_0 \in \cB$, there exists $g \in \partial_{f_0}\msf F$ such that
\begin{equation}
    \forall f \in \cB,\hspace{5mm} \msf F(f) - \msf F(f_0) \le  (f - f_0, g)_{\cB} + \frac{\gamma}{2}||f - f_0||^2_{\cB}
\end{equation}
    
\end{definition}
\begin{remark} Several notions of smoothness exist in the Hilbert space setting~\citep{nesterov}, and these can be appropriately adapted to the Banach space setting. In this work, we focus on \(\gamma\)-smoothness, which is shown to be equivalent to similar smoothness notions as studied in Theorem 3.1 of \cite{wachsmuth2022simpleproofbaillonhaddadtheorem}. In \secref{sec: mainthm}, we examine the theoretical properties of the convergence of the specified mirror descent algorithm (see \secref{subsec: probsetup}) under various assumptions on the underlying loss functional (see \eqnref{eq: optERM}). It is important to note that any theoretical results, whether assuming \(\gamma\)-smoothness or not, apply to the smoothness concepts discussed in \cite{wachsmuth2022simpleproofbaillonhaddadtheorem}.
    
\end{remark}

In the Euclidean setting, the two constants $\mu$ and $\gamma$ can be directly related as studied in \cite{nesterov}. It is straight-forward to note that if a loss functional $\msf F$ is both smooth and strongly convex, then $\mu \le \gamma$.
\begin{lemma}
    If a real-valued functional $\msf F :\cD \subseteq \cB \to \reals$ is $\mu$-strongly convex and $\gamma$-smooth, then $\mu \le \gamma$.
\end{lemma}

Now, we state a useful definition on the boundedness of the Fr\'{e}chet derivative of a functional on $\cB$.
\begin{definition}[$L$-lipshitz]\label{def: lips} We say a convex functional $\msf F: \cD \to \reals$ is $L$-lipshitz w.r.t $||\cdot||_{\cB}$ if for all $f \in \cB$, and subdifferential $g \in \partial_f\msf F$,
\begin{align*}
     ||g||_{\cB^*} \le L
\end{align*}
\end{definition}
\textbf{Polyak-\L{}ojasiewicz inequality}~\citep{POLYAK1963864} has been extensively studied notion in the optimization landscape as a relaxation to convexity. We provide an extension of this notion in the functional setting:
\begin{definition}[Polyak-\L{}ojasiewicz inequality]\label{def: PL}
    A real-valued functional $\msf F : \cD \to \reals$ is called $\mu$-P\L{} if for some $\mu$ > 0:
    \begin{align*}
        \frac{1}{2}||\partial_f \msf F||^2_{\cB^*} \ge \mu(\msf F(f) - \msf F(f^*)), \forall f \in \cB
    \end{align*}
where $f^*$ is a global minimizer of $\msf F$.
\end{definition}


\subsection{Problem Setup and Algorithmic Insights}\label{subsec: probsetup}

Let $\cX$ be the data space and $\cY$ a finite label set. We denote a set of $n$ 
training data points as $D_n := \curlybracket{(x_1,y_1),(x_2,y_2),\ldots,(x_n,y_n)} \subset \cX \times \cY$. 
We consider a modal class $\cB$, a reflexive reproducing kernel Banach space with the goal to find a function $f \in \cB$ such that $f(x_i) \approx y_i$ for every example in $D_n$. 
For a given function $f$ and data point $(\x,y)$, we consider a non-negative penalty function $\ell: \reals \times \reals \to \reals$, where $\ell(f(\x),y)$ is Fr\'{e}chet differentiable wrt to $f$ over entire $\cB$ unless stated otherwise. Using this penalty function, the total loss on the training set $D_n$ is defined via a loss functional over $\cB$, denoted as $\msf L: \cB \to \reals$ as $\msf L(f) = \sum_{i=1}^n \ell(f(\x_i),y_i)$.
We study the following optimization problem with the loss functional $\msf L$ over $\cB$ 
\begin{align}
    \min_{f \in \cB} \msf L(f) \label{eq: optPS}
\end{align}
A well-studied choice of loss functional is square loss for which  \eqnref{eq: optPS} has the form
\begin{align}
    \min_{f \in \cB} \sum_{i=1}^n (f(x_i) - y _i)^2 \label{eq: sqr}
\end{align}
Equivalently, if the loss functional is regularized with a functional $\Psi_{\ge 0}: \cB \to \reals$ with regularization parameter $\lambda > 0$ then we get
\begin{align}
    \min_{f \in \cB} \sum_{i=1}^n (f(x_i) - y _i)^2 + \lambda\cdot \Psi(f) \label{eq: sqrreg}
\end{align}
We solve the optimization problem in \eqnref{eq: optPS} using a \tt{functional form} of mirror descent. 

Consider a strongly convex Fr\'{e}chet differentiable functional $\Phi: \cB \to \reals$, also called a \tt{potential} functional in the literature. For the sake of context, we call them mirror maps in this work. The gradient or the Fr\'{e}chet derivative of $\Phi$ can be thought of as an operator from $\cB$ to $\cB^*$
\begin{align}
    \partial_{(\cdot)} \Phi: \cB \to \cB^* \textnormal{ s.t. } f \mapsto \partial_{f} \Phi \label{eq: mmap}
\end{align}
We discuss mirror maps in great details in \secref{subsec: mirrormap}.
\paragraph{Mirror descent for RKBS $\cB$} Assume $\cB$ is a reflexive reproducing kernel Banach space of real-valued functions over $\cX$. Assume we a mirror map $\Phi: \cB \to \reals$.

We optimize the general minimization problem in \eqnref{eq: optPS} for a given loss functional $\msf L$ with the following mirror descent algorithm (MDA)
\begin{subequations}\label{eq: prelmd}
\begin{align}
    g_t &:= g_{t-1} - \eta\cdot\partial_{f_{t-1}} \msf L \\
    f_t &:= (\partial \Phi)^{-1} (g_t)
\end{align}
\end{subequations}
where $g_i \in \cB^*$ and $\partial_{f_{t-1}} \msf L$ is the Fr\'{e}chet derivative of a loss functional $\msf L: \cB \to \reals$ with respect to $f_{t-1}$. 

We denote the corresponding reproducing kernel to the RKBS $\cB$ as $K: \cX \times \cB \to \reals$. Consider an input ${\x}$ and the evaluation functional $\delta_{\x} : \cB \to \reals$.
Now, the G\^{a}teaux differential of $\delta_{\x}(f)=f({\x})$ w.r.t $f$ at any function $h \in \cB$ is given by
\begin{equation}
    \partial_f (\delta_{\x})[h] = \left\langle h, K({\x},\cdot)\right\rangle_{\cB} \label{eq: difeval}
\end{equation}

Assuming that the loss functional $\msf L$ is square loss~(see \eqnref{eq: sqr}), we can compute the Fr\'{e}chet derivative of $\msf L$ wrt $f$ as follows:
\begin{align}
    \partial_f \msf L(\cdot) = \partial_f \paren{\sum_{i = 1}^n (f(\x_i) - y_i)^2} &= 2\sum_{i=1}^n \paren{f(\x_i) - y_i}\cdot \partial_f f(\x_i)\nonumber\\
    &= 2\sum_{i=1}^n \paren{f(\x_i) - y_i}\cdot \partial_f \dist{f}{K(\x_i,\cdot)}\nonumber\\
    &= 2\sum_{i=1}^n \paren{f(\x_i) - y_i}\cdot K(\x_i,\cdot) \label{eq: sqdiff}
\end{align}
where the last equation follows using \eqnref{eq: difeval}. Note that the terms $K(\x_i,\cdot)$ for any $i$ is in the dual space $\cB^*$, and thus the summation is in the dual space.

With this, we provide an \tt{iterative algorithm} using the kernel $K$:
\begin{subequations}\label{eq: MD}
\begin{align}
    g_t &\gets g_{t-1} - 2\eta \sum_{i=1}^n (f_{t-1}(\x_i) - y_i)\cdot K(\x_i,\cdot)  \\
    f_{t} &\gets \paren{\partial \Phi}^{-1} (g_t)
\end{align}
\end{subequations}
Notice that if $g_0$ is initialized as $\sum_{i=1}^n \bc^{(0)}_i\cdot K(\bx_i,\cdot)$ for some $\bc^{(0)} \in \reals^n$ then inductively every $g_t$ has the form 
\begin{subequations}\label{eq: instantMD}
\begin{align*}
    \bc^{(t)} &\gets \bc^{(t-1)} - 2\eta(f_{t-1}(\x) - Y)\\ 
    g_t &= \sum_{i = 1}^n \bc^{(t)}_i \cdot K(\x_i,\cdot)\\
    f_{t} &\gets \paren{\partial \Phi}^{-1} (g_t)
\end{align*}
\end{subequations}
where we use the vectorial notation $f_i(\x) = (f_i(\x_1), f_i(\x_2),\ldots, f_i(\x_n))$, $Y = (y_1,y_2,\ldots,y_n)$. Thus, we can make updates to a vector $\bc^{(t)} \in \reals^n$ to keep track of gradients in the dual space.
Similary, the MDA can be stipulated in the regularized learning in \eqnref{eq: sqrreg} where we assume that
the mirror map is same as the regularization term $\Psi$ for simplification, i.e. $\Phi = \Psi$
\begin{subequations}\label{eq: MDreg}
\begin{align}
    g_t &\gets (1 - 2\eta \lambda)\cdot g_{t-1} - 2\eta \sum_{i=1}^n (f_{t-1}(\x_i) - y_i)\cdot K(\x_i,\cdot)  \\
    f_{t} &\gets \paren{\partial \Phi}^{-1} (g_t)
\end{align}
\end{subequations}
where we have used the fact that $\partial_f \Phi = \partial_f \Psi$.
Similarly to the nonregularized setting, this could be further simplified to
\begin{subequations}
\begin{align}
    \bc^{(t)} &\gets (1 - 2\eta \lambda)\bc^{(t-1)} - 2\eta(f_{t-1}(\x) - Y)\\ 
    g_t &= \sum_{i = 1}^n \bc^{(t)}_i \cdot K(\x_i,\cdot)\\
    f_{t} &\gets \paren{\partial \Phi}^{-1} (g_t)
\end{align}
\end{subequations}
Recent works have established representer theorems for different settings in \eqnref{eq: optPS}. For example, \citet{zhang09b} and \citet{Lin2022} studied regularized learning in \eqnref{eq: sqrreg} for semi-inner product (s.i.p) RKBSs with continuous and convex loss functionals \(\msf L\). They showed that the optimal classifier \(f\) in the primal Banach space has a dual representation which can be uniquely written as \(g^* = \sum_{i = 1}^n c_i^* K(\x_i, \cdot)\) for scalars \(c_i^* \in \mathbb{R}\). This emphasizes the novelty of the MDA iterations, as they could potentially achieve the optimal solution to \eqnref{eq: optPS}. 

\begin{remark}
\eqnref{eq: MDreg} can be easily extended to cases where the Fr\'{e}chet derivatives of \(\Phi\) and \(\Psi\) map to scaled linear transformations in \(\mathcal{B}^*\). Thus, MDA can be applied to a natural choice studied in the literature where \(\Psi = \psi(\|\cdot\|_{\mathcal{B}})\) and \(\Phi = \frac{1}{2}\|\cdot\|_{\mathcal{B}}^2\) (or \(\phi(\|\cdot\|_{\mathcal{B}}^2)\)). However, we can eliminate these assumptions for certain families of RKBSs where the functions are induced by the training points. In such cases, every \(g_t \in \mathcal{B}^*\) is induced by \(\{K(\bx_i, \cdot)\}\), e.g., \(p\)-norm RKBSs as constructed in \secref{sec: example}.
\end{remark}

 In \secref{sec: example}, we discuss the construction of an RKBS for which one can stipulate explicit primal and dual space updates for a suitable choice of a mirror map.

Note that MDA for square loss in \eqnref{eq: MD} and \eqnref{eq: MDreg} could be generalized for \tt{any} differentiable loss functional $\msf L$. The iteration step includes the appropriate Fr\'{e}chet derivative of the loss functional 
\begin{align*}
    \partial_f \msf L(f, \curlybracket{(\x_i,y_i)}_{i=1}^n) = \sum_{i=1}^n \frac{\partial \msf \ell(z,y_i)}{\partial z} \Big\lvert_{z = f(\x_i)}\cdot \partial_f f(\x_i) = \sum_{i=1}^n \frac{\partial \msf \ell(z,y_i)}{\partial z} \Big\lvert_{z = f(\x_i)}\cdot K(\x_i,\cdot)
\end{align*}

A natural question is \tt{if the mirror descent algorithm of \eqnref{eq: prelmd} is statistically efficient}? We study this in \secref{sec: mainthm}. This requires appropriate choices of mirror maps to achieve certain convergence guarantees. Below, we provide formal treatment of these maps.

\subsection{Mirror Maps}\label{subsec: mirrormap}
Consider a convex open set of functionals $\cC$ such that $\cB \subset \Bar{\cC}$. We are interested in certain real-valued functionals on $\cC$ for the mirror descent algorithm of \eqnref{eq: MD} that could be used to map functions in the primal space $\cB$ to transformations in $\cB^*$. 

In \secref{subsec: probsetup} we introduced mirror maps to discuss the mirror descent algorithms in \eqnref{eq: prelmd}. Formally, we define mirror maps as follows:

\begin{definition}[Mirror Map]\label{def: mp} We say that a functional $\Phi: \cC \to \reals$ is a mirror map if it satisfies the following properties:
\begin{enumerate}[leftmargin=3em]
    \item $\Phi$ is strictly convex. 
    \item Subgradient sets of $\partial_{(\cdot)}\Phi$ don't intersect at non-empty set and $\partial_{(\cdot)} \Phi (\cB) = \cB^*$.
\end{enumerate}
\end{definition}

The Condition \tt{2.} above is the key ingredient to the MDA in \eqnref{eq: MD}. The injectivity of subgradient sets makes sure that the algorithm can converge without getting stuck in a loop and the surjectivity makes sure that the algorithm gets back to the primal space reliably.
The algorithm would make sense only if the chosen mirror map satisfies this property. We can guarantee this for a wide variety of strictly convex functionals on a reflexive Banach space, including squared $p$-norms.
Note that strict convexity of $\Phi$ (in Condiiton $1.$) implies that the subgradient sets at any $f,f' \in \cB$ don't intersect, i.e.
\begin{align*}
    \partial_f \Phi \cap  \partial_{f'} \Phi = \emptyset
\end{align*}
On the other hand, Condition \tt{2.} implies that there exists some $f$ such that for any $g \in \cB^*$, $g \in \partial_f \Phi$. In \lemref{lemma: inversemirror}, we establish a strong result on the injectivity and subjectivity of a mirror map.

We assume that the underlying functional for proper, \tt{i.e.}
$\msf F: \cB \to \reals$ such that $(\textnormal{dom } \msf F) \neq \emptyset$ and $\msf F(f) \ge - \infty$ for all $f \in \cB$. We defer the proof of the lemma to \appref{app: technical}.

\begin{lemma}\label{lemma: inversemirror}
    Consider a convex Banach space $(\cB, \norm{\cdot}_{\cB})$. Let $\msf F: \cB \to \reals$ be a proper, strictly convex and G\^ateaux differentiable functional. 
    Then, for all $f, f' \in \cB$,
    \begin{align}
        \partial_{f} \msf F = \partial_{f'} \msf F \implies f = f'
    \end{align}
    Furthermore, for any linear functional $\bg{}^* \in \cB^*$, there exists $\hat{f} \in \cB$ such that
    \begin{align*}
        \bg{}^* = \partial_{\hat{f}} \msf F
    \end{align*}
    In order words, the operator $\partial_{(\cdot)} \msf F: \cB \to \cB^*$ is both injective and subjective, where 
    \begin{align*}
        \forall f \in \cB,\,\, f \longmapsto \partial_{(\cdot)} \msf F(f) := \partial_f \msf F \in \cB^*
    \end{align*}
\end{lemma}

Although, we have stated the lemma for a G\^ateaux  differentiable mirror app, subjectivity can be achieved in its absence (see the proof in \appref{app: technical}).

\section{Theoretical results: Convergence of Mirror Descent Algorithm}\label{sec: mainthm}

In this subsection, we consider the unconstrained optimization problem 
\begin{equation}
    \min_{f \in \cB} \msf L(f) \label{eq: global}
\end{equation}
over a Banach space $\cB$ and a real loss functional $\msf L: \cB \to \reals$. We would like to understand how well the mirror descent algorithm of \eqnref{eq: MD} performs if there exists a realizable global minimizer of \eqnref{eq: global}, \tt{i.e.} there exists $f^* \in \cB$ such that $\argmin_{f \in \cB} \msf L(f) = \curlybracket{f^*}$. However, the existence of a global minimizer in the Banach space for \eqnref{eq: global} is not guaranteed, a situation that can be resolved by assuming the reflexivity of the space (see \citet[Theorem 2.3.1]{Zlinescu2002ConvexAI}).

In this section, we assume that the Banach space \(\cB\) is reflexive. This assumption serves two purposes: first, it ensures the existence of a global minimizer for \eqnref{eq: global} within the space; second, it allows us to study the convergence properties of the mirror descent algorithm under the assumption that the underlying functionals—either the loss functional \(\msf L\) or the mirror map \(\Phi\)—are strongly convex (see \defref{def: strongconv}). The second condition also implies the reflexivity of the space, as stated in the following result.

\begin{theorem}[Theorem 3.5.13~\cite{Zlinescu2002ConvexAI}]
Let \(\cB\) be a Banach space. If there exists a proper, lower semi-continuous, uniformly convex functional \(\msf F :\cB \to \reals\) whose domain has a nonempty interior, then \(\cB\) is reflexive.
\end{theorem}

\subsection{Non-existence of a smooth and strongly convex functional}\label{sec: non-existence}
In the Euclidean setting, the convergence rate for various gradient-based methods applied to the optimization problem 
\begin{align}
   \min_{x \in \mathbb{R}^d} f(x), \label{eq: euclid}
\end{align}
where \( f: \mathbb{R}^d \to \mathbb{R} \), has been extensively studied. To achieve a linear convergence rate with gradient descent \citep{bubeck2015convex}, a common requirement is that the loss function \( f \) is both strongly convex and smooth. A natural question arises: \tt{can we achieve a linear rate for the optimization problem in \eqnref{eq: global} within reflexive Banach spaces?}

The answer is likely negative for Banach spaces that are not isomorphic to a Hilbert space. We can show that there does not exist a functional \( F: \mathcal{B} \to \mathbb{R} \) that is both strongly convex and \(\gamma\)-smooth on a Banach space that is not isomorphic to a Hilbert space.

We state this negative result regarding the existence of a strongly convex and \(\gamma\)-smooth functional for general Banach spaces as follows. The proof is deferred to \appref{app: smooth-convex}.
\begin{lemma}[Existence of strongly convex and \(\gamma\)-smooth functional]\label{lemma: existence}
    Let \(\mathcal{B}\) be a Banach space. If there exists a functional \( F: \mathcal{B} \to \mathbb{R} \) that is both \(\mu\)-strongly convex and \(\gamma\)-smooth for some \(\mu > 0\) and \(\gamma < \infty\), then \(\mathcal{B}\) is isomorphic to a Hilbert space.
\end{lemma}

The proof is based on the characterization of second-order differentiable points of a functional as discussed in \cite{Borwein1994SecondOD}. Using these points, we can demonstrate the existence of a Hilbert norm on any separable Banach space. By applying a generalization of the parallelogram law to show the isomorphism of a Banach space onto a Hilbert space (see \cite{Kwapień1972}), the proof extends to general Banach spaces by considering their separable subspaces.


\subsection{Conditional linear rate for unconstrained optimization}
In the previous section (\ref{sec: non-existence}), we demonstrated that there exist Banach spaces that do not admit functionals which are both strongly convex (see Definition \ref{def: strongconv}) and smooth (see Definition \ref{def: smooth}). Here, we consider a specific class of RKBSs that satisfy a certain property, which we term `Hilbertizable', and show a linear rate of convergence for unconstrained optimization using the mirror descent algorithm. This rate is achieved for loss functionals that are both strongly convex and \(\gamma\)-smooth.

In Definition \ref{def: mp}, we require the mirror map to be strictly convex. To establish convergence guarantees, we further need slightly stronger conditions, namely \(\gamma\)-smoothness and \(\mu\)-strong convexity.

First, we state an assumption regarding the specific class of Banach spaces that could potentially admit strongly convex and smooth functionals.\vspace{-1mm}
\begin{assumption}\label{ass: 1} 
A Banach space \(\mathcal{B}\) is termed Hilbertizable if \((\mathcal{B}, \|\cdot\|_{\mathcal{B}})\) can be isomorphically mapped onto a Hilbert space.  
\end{assumption}

Although topologically \((\mathcal{B}, \|\cdot\|_{\mathcal{B}})\) is same as a Hilbert space, but note that this is a weaker notion than isometric isomorphism, which, if it exists, implies that the Banach space norm \(\|\cdot\|_{\mathcal{B}}\) is a Hilbert norm. Finite sequences with $\ell_p$ (where $p \in (1,\infty)$) norm are isomorphic to each other. On the other hand, square-summable infinite sequences $\ell^2(\nats)$ with $\ell_2$ and $\ell_2 + \ell_{\infty}$ norms are isomorphic (but not isometrically) to each other, where one is a Hilbert space and the other Banach space.

Now, we state another assumption that characterizes the existence of a smooth and strongly convex functional.\vspace{-1mm}
\begin{assumption}\label{ass: 2} 
A reflexive Banach space \(\mathcal{B}\) is termed smoothly-convex if \((\mathcal{B}, \|\cdot\|_{\mathcal{B}})\) admits a functional \(\mathsf{F}: \mathcal{B} \to \mathbb{R}\) which is both \(\mu\)-strongly convex and \(\gamma\)-smooth (with respect to \(\|\cdot\|_{\mathcal{B}}\)) for \(\mu > 0\) and \(\gamma < \infty\).
\end{assumption}

Thus, for a Banach space to satisfy Assumption \ref{ass: 2}, Lemma \ref{lemma: existence} states that it must satisfy Assumption \ref{ass: 1}.

Under the aforementioned assumptions on the underlying Banach space, we state and prove the linear rate of MDA as follows:


\begin{theorem}[unconstrained optimization]\label{thm: uncons}Consider the optimization problem:
\begin{align*}\label{eq:min_unconstrained}
\min_{f \in \cB} \msf L(f)
\end{align*}
where $\cB$ is a reflexive RKBS that satisfies \assref{ass: 1} and \assref{ass: 2}. Assume that the loss functional $\msf L$ is $\mu$-strongly convex and $\gamma$-smooth (w.r.t $||\cdot||_{\cB}$). Furthermore, assume that there exists a mirror map $\Phi$ that is $\nu$-strongly convex and $\rho$-smooth (w.r.t. $||\cdot||_{\cB}$). Let $f^* \in \cB$ be the unique global minima of the optimization objective, then the mirror descent algorithm of \eqref{eq: MD} converges to the optimal solution in $\cB$ with the learning rate $\eta = \min \curlybracket{\frac{\nu}{\gamma}, \frac{1}{2\mu\nu\kappa^2}}$. Moreover, the convergence rate is linear, \tt{i.e.},
\begin{align*}
   \msf L(f_k) - \msf L(f^*) \le (\msf L(f_{0}) - \msf L(f^*))\cdot e^{-k\cdot\frac{\mu\nu^2\kappa^2}{\gamma}} 
\end{align*}
where $\kappa$ depends on $\rho$.
\end{theorem}

\begin{proof}{\bf{of} \thmref{thm: uncons}} In the remainder of the proof, let \( f_k \) denote the \( k \)-th update in the primal space \(\cB\), and let \(\bg{k} := \partial_{f_k} \Phi\). Since \(\msf L\) or \(\Phi\) may not be G\^{a}teaux differentiable, we use \(\partial_{f} \msf L\) or \(\partial_{f} \Phi\) to represent their subgradient sets for any \( f \in \cB \). In the mirror descent algorithm (MDA), we select one element randomly from these subgradient sets (to resolve ties). So, without loss of generality, we write these sets from the chosen elements. Note that the updates $\bg{k}$ and $\bg{k-1}$  in the dual space are related according to the \MDA~as follows:
\begin{align}
    \bg{k} := \bg{k-1} - \eta\cdot \partial_{f_{k-1}}\msf L
\end{align}

Consider the updates $f_k$ and $f_{k-1}$. Using the strong convexity of $\Phi$, we have
\begin{align*}
    &\Phi(f_k) \ge \Phi(f_{k-1}) + \left\langle f_k - f_{k-1},\partial_{f_{k-1}} \Phi\right\rangle_{\cB} + \frac{\nu}{2}||f_k - f_{k-1}||^2_{\cB}\\
    &\Phi(f_{k-1}) \ge \Phi(f_{k}) + \left\langle f_{k-1} - f_{k},\partial_{f_{k}} \Phi\right\rangle_{\cB} + \frac{\nu}{2}||f_k - f_{k-1}||^2_{\cB}
\end{align*}
Adding the equations above we get
\begin{align}
    0 &\ge \left\langle f_k - f_{k-1}, \partial_{f_{k-1}}\Phi - \partial_{f_{k}}\Phi \right\rangle_{\cB} + \nu\cdot ||f_k - f_{k-1}||_{\cB}^2 \notag\\
    &= \left\langle f_k - f_{k-1}, \bg{k-1} - \bg{k}\right\rangle_{\cB} + \nu\cdot ||f_k - f_{k-1}||_{\cB}^2 \notag\\
    &= \eta\cdot \left\langle f_k - f_{k-1}, \partial_{f_{k-1}} \msf L\right\rangle_{\cB} + \nu\cdot ||f_k - f_{k-1}||_{\cB}^2 \label{eq: upbound}
\end{align}
\eqnref{eq: upbound} is useful in the sense that it provides a way to bound $||f_k - f_{k-1}||_{\cB}^2$ in terms of $\partial_{f_{k-1}} \msf L$. We would show that the norm of $||\partial_{f_{k-1}} \msf L||_{\cB^*}$ could be bounded the other way round in terms of $||f_k - f_{k-1}||_{\cB}^2$. In order to show that we need the following lemma that establishes the connection of the mirror map $\Phi$ to its convex conjugate defined on the dual space.

\begin{proposition}[{\cite[Corollary 3.5.7]{Zlinescu2002ConvexAI}}]\label{lemma: dual_convex}
    Let $\msf F: \cB \to \reals$ be a continuous convex function and $p,q \in \reals$ be such that $1 \le p \le 2 \le q$ and $p^{-1} + q^{-1} = 1$. Then, the following statements are equivalent
    \begin{enumerate}
        \item $\exists L_2 > 0$, $\forall f,f' \in \cB$, $\forall\, g \in \partial_{f'} \msf F$:
        \begin{align*}
          \msf F(f) \le \msf F(f') + \left\langle f - f', g\right\rangle_{\cB} + \frac{L_2}{p}\cdot||f - f'||^p_{\cB};  
        \end{align*}
        \item $\exists L_5 > 0$, $f, f' \in \cB$,
        \begin{align*}
          \left\langle f - f', \partial_{f} \msf F - \partial_{f'} \msf F\right\rangle_{\cB} \ge \frac{2}{L_5 q}\cdot||\partial_{f} \msf F - \partial_{f'} \msf F||^q_{\cB^*}.  
        \end{align*}
    \end{enumerate}
\end{proposition}

We could apply \lemref{lemma: dual_convex} for the mirror map $\Phi$. Since it is convex and $\rho$-smooth the condition \tt{1.} holds and thus, there exists a scalar $\kappa > 0$ such that $\kappa := \frac{1}{L_3}$ for which condition \tt{2.} is satisfied for any functions $f,f' \in \cB$. 

Using \lemref{lemma: dual_convex}, for the iterates $f_k$ and $f_{k-1}$ we have 
\begin{align}
    \left\langle f_k - f_{k-1}, \partial_{f_k} \Phi - \partial_{f_{k-1}} \Phi \right\rangle_{\cB} &\ge 
    ||\partial_{f_k} \Phi - \partial_{f_{k-1}} \Phi||^2_{\cB^*} \notag\\
      \implies \eta\cdot\left\langle f_k - f_{k-1}, - \partial_{f_{k-1}} \msf L\right\rangle_{\cB} &\ge \kappa\eta^2\cdot||\partial_{f_{k-1}} \msf L||^2_{\cB^*} \label{eq: lowbound}
\end{align}
In the equation above we note that $\partial_{f_k} \Phi = \bg{k}$ for all $k = 0, 1,\ldots$. Thus, $\partial_{f_k} \Phi - \partial_{f_{k-1}} \Phi = - \eta\cdot \partial_{f_{k-1}} \msf L$. 

But using Cauchy-Schwartz inequality, we note that 
\begin{align*}
    \left\langle f_{k-1} - f_k, \partial_{f_{k-1}} \msf L\right\rangle_{\cB^*} \le ||f_{k-1} - f_k||_{\cB} \cdot ||\partial_{f_{k-1}} \msf L||_{\cB^*}
\end{align*}
Thus, we can write
\begin{align}
    ||f_k - f_{k-1}||_{\cB} \ge \kappa\eta\cdot||\partial_{f_{k-1}} \msf L||_{\cB} \label{eq: lowbound}
\end{align}
Now, using $\gamma$-smoothness of $\msf L$ we note that
\begin{align}
    \msf L(f_k) &\le \msf L(f_{k-1}) + \left\langle f_k - f_{k-1}, \partial_{f_{k-1}} \msf L\right\rangle_{\cB} + \frac{\gamma}{2}||f_k - f_{k-1}||_{\cB}^2 \notag\\
    &\le \msf L(f_{k-1}) + \paren{\frac{\gamma}{2} - \frac{\nu}{\eta}}\cdot||f_k - f_{k-1}||_{\cB}^2 \label{eq: up1}\\
    &\le \msf L(f_{k-1}) - \paren{\frac{\nu}{\eta} - \frac{\gamma}{2}}\cdot \kappa^2\eta^2 \cdot ||\partial_{f_{k-1}} \msf L||^2_{\cB^*} \label{eq: up2}\\
    &= \msf L(f_{k-1}) - \frac{(2\nu - \gamma\eta)\kappa^2\eta}{2} \cdot ||\partial_{f_{k-1}} \msf L||^2_{\cB^*} \label{eq: boundG}
\end{align}
In \eqnref{eq: up1}, we bound $\left\langle f_k - f_{k-1}, \partial_{f_{k-1}} \msf L\right\rangle_{\cB}$ in terms of $||f_k - f_{k-1}||_{\cB}^2$ using \eqnref{eq: upbound}. But then $||f_k - f_{k-1}||_{\cB}^2 $ can be bounded in terms of $||\partial_{f_{k-1}} \msf L||_{\cB^*}$ using \eqnref{eq: lowbound}.
Here, note that $\frac{2\nu}{\gamma} > \eta$.

Now, we would try to lower bound $\msf L(f^*)$ using the definition of the minimizer.
Consider the following:
\begin{align}
    \msf L(f^*) &:= \inf_{f \in \cB} \msf L(f) \notag\\
    & = \inf_{f \in \cB} \msf L(f_k + f) \notag\\
    & \ge \inf_{f \in \cB} \msf L(f_k) + \left\langle f, \partial_{f_k} \msf L\right\rangle_{\cB^*} + \frac{\mu}{2}\cdot||f||^2_{\cB} \notag\\
    & = \msf L(f_k) + \mu\cdot \paren{\inf_{f \in \cB} -\left\langle f, -\frac{1}{\mu}\cdot \partial_{f_k} \msf L\right\rangle_{\cB} + \frac{1}{2}||f||_{\cB}^2 } \notag\\
    & = \msf L(f_k) - \frac{1}{2\mu}\cdot ||\partial_{f_{k-1}} \msf L||^2_{\cB^*} \label{eq: boundmin}
\end{align}
In the equation above, we use the fact that the convex conjugate of  
$\frac{1}{2}||\cdot||^2_{\cB}$ is $\frac{1}{2}||\cdot||^2_{\cB^*}$.

Now, we could bound $\msf L(f_k) - \msf L(f^*)$ geometrically. Using \eqnref{eq: boundG} and \eqnref{eq: boundmin}
\begin{align*}
    \msf L(f_k) &\le \msf L(f_{k-1}) - \frac{(2\nu - \gamma\eta)\kappa^2\eta}{2}\cdot||\partial_{f_{k-1}} \msf L||_{\cB^*}\\
    &\le \msf L(f_{k-1}) + \mu(2\nu - \gamma\eta)\kappa^2\eta\cdot\paren{\msf L(f^*) - \msf L(f_{k-1})}
\end{align*}

Now, this could be simplified as follows:
\begin{align*}
  \msf L(f_k) - \msf L(f^*)  &\le -\paren{\msf L(f^*) - \msf L(f_{k-1})} + \mu(2\nu - L\eta)\kappa^2\eta\cdot\paren{\msf L(f^*) - \msf L(f_{k-1})}\\
  &\le \paren{1 - \mu(2\nu - \gamma\eta)\kappa^2\eta}(\msf L(f_{k-1}) - \msf L(f^*))\\
  &\vdots\\
  &= \paren{1 - \mu(2\nu - \gamma\eta)\kappa^2\eta}^{k}(\msf L(f_{0}) - \msf L(f^*))\\
  &\le \lambda \cdot \exp\left(-k\cdot\mu(2\nu - \gamma\eta)\kappa^2\eta\right)\\
  &= \lambda\cdot \exp\left(-k\cdot\frac{\mu\nu^2\kappa^2}{\gamma}\right)
\end{align*}
where we define constant $\lambda := (\msf L(f_{0}) - \msf L(f^*))$. With this, we have completed the proof.
\end{proof}

\paragraph{Extension to P\L{} condition on the loss functional} Now, we consider relaxing the strong convexity condition (see \defref{def: strongconv}) on the loss functional to establish \thmref{thm: PL}. In \defref{def: PL}, we introduced the Polyak-\L{}ojasiewicz (P\L{}) condition for the functional setup.

The P\L{} inequality~\citep{POLYAK1963864, lojasiewicz1963topological} has been extensively studied as a weaker condition on a loss function compared to strong convexity in minimization problems within the Euclidean setting (see \eqnref{eq: euclid}). This condition is commonly encountered in applications and often leads to fast convergence of numerical algorithms. Under the P\L{} condition, a linear rate of convergence for gradient descent in the classical setting above can be demonstrated~\citep{PL_linear}.

Thus, it is natural to explore this in the functional setup. Indeed, it is sufficient for the convergence of the MDA. The rate of convergence remains linear if we replace the $\mu$-strong convexity of the loss functional $\msf L$ with the $\mu$-P\L{} condition. We state the result as follows:


\begin{theorem}[Unconstrained Optimization under $\mu$-P\L{} Condition]\label{thm: PL}
Consider the optimization problem:
\begin{align}\label{eq:min_unconstrained1}
\min_{f \in \cB} \msf L(f),
\end{align}
where $\cB$ is a reflexive RKBS that satisfies \assref{ass: 1} and \assref{ass: 2}. Assume that the loss functional $\msf L$ is $\mu$-P\L{} and $\gamma$-smooth (with respect to \( ||\cdot||_{\cB} \)). Furthermore, assume that there exists a mirror map $\Phi$ which is $\nu$-strongly convex and $\rho$-smooth (with respect to \( ||\cdot||_{\cB} \)). Let \( f^* \in \cB \) be the unique global minimizer of the optimization objective. Then, the mirror descent algorithm described by \eqref{eq: MD} converges to the optimal solution in \( \cB \) with the learning rate \( \eta = \min \left\{\frac{\nu}{\gamma}, \frac{1}{2\mu\nu\kappa^2} \right\} \). Moreover, the convergence rate is linear, i.e.,
\begin{align*}
   \msf L(f_k) - \msf L(f^*) \le (\msf L(f_{0}) - \msf L(f^*)) \cdot e^{-k \cdot \frac{\mu\nu^2\kappa^2}{\gamma}},
\end{align*}
where \( \kappa \) depends on \( \rho \).
\end{theorem}
The proof of this theorem follows as before, with the addition that \eqnref{eq: boundmin} holds due to the P\L{} condition of the loss functional $\msf L$.

\subsection{Constrained Optimization: Convergence of projected MDA}\label{sec: unconstrained}
In the previous sections, we study the conditions under which MDA can achieve linear rate of convergence for unconstrained optimization over a reproducing kernel Banach space. Having demonstrated this linear rate in the unconstrained setting, an important question is how MDA, as expressed in \eqnref{eq: MD}, behaves when used in a constrained setting. In this scenario, the algorithm is restricted to choosing functions solely from a constrained subset of the Banach space $\mathcal{B}$. 

We establish that under mild assumptions (not requiring \assref{ass: 1} and \assref{ass: 2}) on the mirror map and the loss functional, a rate of $\mathcal{O}\left(\frac{1}{\sqrt{t}}\right)$ can be attained. This paradigm has undergone extensive examination in the classical setting; \cite{bubeck2015convex} explored various algorithms employing mirror descent in Banach spaces defined over Euclidean spaces, and with due consideration, these results can be extended to the functional case as well.

\paragraph{Mirror Descent for constrained optimization} Let $\cB_0 \subseteq \cB$ be a compact convex set of functionals. Consider a loss functional $\msf F: \cB \to \reals$. We are interested in the following optimization problem:\vspace{-2mm}
\begin{align}
    \min_{f\in \cB_0} \msf L(f) \label{eq: consopt}
\end{align}
Now, note that if we try to run the mirror descent algorithm of \eqnref{eq: MD}, then the preimage of $g_t$, \tt{i.e.} $(\partial \Phi)^{-1}(g_t)$ may not be in $\cB_0$. A simple solution to this issue could be to project the preimage back to $\cB_0$. In order to do this, we study the notion of a Bregman divergence.

Consider a real-valued functional $\msf A : \cB \to \reals$. We define the \tt{Bregman divergence} $\mf D_{\sf A}$ \tt{wrt} $\msf A$ as follows:
\begin{align}
    f,f' \in \cB,\quad \mf D_{\sf A}(f,f') := \sf A(f) - \sf A(f') - \left\langle f-f', \partial_{f'}\sf A\right\rangle_{\cB} \label{eq: BD}
\end{align}
This gives a straight-forward identity--- for any $f,f',h \in \cB$
\begin{align}
    \dist{f - h}{\partial_{f}\sf A - \partial_{f'}\sf A} = \mf D_{\sf A}(f,f') + \mf D_{\sf A}(h,f) - \mf D_{\sf A}(h,f') \label{eq: bregid}
\end{align}
Consider a convex set of functionals $\cC \subseteq \cB$ such that closure of $\cC$ contains $\cB_0$, \tt{i.e.} $\cB_0 \subseteq \Bar{\cC}$. We consider mirror maps that satisfy Condition 2. in \defref{def: mp} for the set $\cC$. 
\begin{definition}[Projection Map]\label{def: pmap} For a mirror map $\Phi: \cC \to \reals$, we define a projection map $\Pi_{\cB_0}^{\Phi}$ for any $f' \in \cC$ as
\begin{align*}
    \Pi_{\cB_0}^{\Phi}(f') := \argmin_{f \in \cB_0 \cap \cC} \mf D_{\Phi}(f,f')
\end{align*}
\end{definition}
There are pathological cases where the rhs $\argmin$ is an empty set. To avoid such cases we further assume that $\Phi$ diverges on the boundary of $\cC$, i.e. $\lim_{f \to \partial \cC} \norm{\partial_f \Phi}_{\cB^*} = + \infty$, thus it is coersive.

With this definition, we propose \tt{projected} mirror descent algorithm to solve the optimization problem in \eqnref{eq: consopt} as follows. Assuming that the loss function $\msf L$ is squared error, \tt{i.e.} $\msf L(f) = \sum_{i=1}^n (f(x_i) - y_i)^2$, we provide a projected iterative algorithm using a reproducing kernel $K$ for $\cB$ and a projection map $\Pi_{\cB_0}^\Phi$ (see \defref{def: pmap}) as follows: 
\begin{subequations}\label{eq: consMDA}
\begin{align*}
    g_t &\gets g_{t-1} - 2\eta \sum_{i=1}^n (f_{t-1}(x_i) - y_i)\cdot K(\cdot,x_i)\\
    f_{t} &\gets \Pi_{\cB_0}^\Phi \paren{\paren{\partial \Phi}^{-1} (g_t)}
\end{align*}
\end{subequations}
In the iteration step, we haven't assumed G\^ateaux differentiability of $\Phi$, thus we treat $\partial \Phi$ as a subgradient set. Since a strictly convex functional has non-intersecting subgradient sets, the projection is well-defined. Furthermore, if $g_0$ is initialized as $\sum_{i=1}^n \bc^{(0)}_i\cdot K(\bx_i,\cdot)$ for some $\bc^{(0)} \in \reals^n$ then inductively every $g_t$ has a form similar to \eqnref{eq: instantMD}. Although the algorithm requires a reproducing kernel from computational point of view, we can show that the projected algorithm would converge even under mild assumptions on a general loss functional $\msf L$. For the rest of the section, we would consider the following iteration steps to avoid the explicit existence of a reproducing kernel for a reflexive Banach space $\cB$.
\begin{subequations}\label{eq: constMD}
\begin{align}
    g_t &\gets g_{t-1} - \eta\cdot \partial_{f_{t-1}} \msf L\\
    f_{t} &\gets \Pi_{\cB_0}^\Phi \paren{\paren{\partial \Phi}^{-1} (g_t)}
\end{align}
\end{subequations}
Rest of the section, we will show the convergence of the steps above. First, we show a useful result on the optimality of the projection in the projected mirror descent algorithm with the proof detailed in \appref{app: natdirection}. This essentially states that even after the projection the algorithm always points in the direction of descent. 

\begin{lemma}\label{lemma: natdirection}
 Let \( f \in \cB_0 \cap \cC \) and \( f' \in \cC \), then
 \begin{align*}
     \inner{\Pi_{\cB_0}^{\Phi}(f') - f,\hspace{1mm} \partial_{\Pi_{\cB_0}^{\Phi}(f')} \Phi - \partial_{f'} \Phi}_{\cB} \le 0
 \end{align*}
\end{lemma}

We demonstrate convergence for a class of loss functionals that are \(L\)-Lipschitz (\defref{def: lips}) and convex. Using the computation of Frechet derivative in \eqnref{eq: sqdiff}, we can show that square and cross-entropy loss functionals satisfy this property if the kernel $K$ and functions $f \in \cB$ are bounded. Additionally, the mirror map is required to be \(\mu\)-strongly convex. A notable example of such mirror maps is the squared \(p\)-norms (\(1 < p < 2\)) on finite-dimensional Banach spaces. In Section \ref{sec: example}, we construct a \(p\)-norm RKBS and provide a concrete algorithm for mirror maps of the form \(\frac{1}{2}\|\cdot\|_p^2\) (see Algorithm \ref{alg: md}).

Now, we state the main result of this subsection on the convergence of projected MDA in \eqref{eq: constMD}, with the proof detailed below. In the discussion above, we did not assume \(\Phi\) to be Gâteaux differentiable; however, for ease of discussion and without loss of generality, we assume it in the following statement.
\begin{theorem}[constrained optimization]\label{thm: constopt}
Assume $\cB$ be a reflexive Banach space. Let $\cB_0 \subseteq \cB$ be a compact convex set of functionals. Consider the optimization problem:
\begin{align*}
    \min_{f \in \cB_0} \msf L(f)
\end{align*}
Assume that the loss functional $\msf L$ is convex and $L$-Lipshitz (wrt $\norm{\cdot}_\cB$). Furthermore, assume that there exists a proper mirror map $\Phi: \cC \to \reals$ that is G\^ateaux differentiable and $\mu$-strongly convex over $\cC \cap \cB_0$ (wrt $\norm{\cdot}_\cB$), where $\cC$ is the closure of $\cB_0$. Let $R^2 = \max_{f, f' \in \cB_0 \cap \cC} |\Phi(f) - \Phi(f')|$. Then, there is a choice of $\eta$ such that the projected mirror descent algorithm of \eqnref{eq: constMD} converges to the optimal solution in $\cB_0$ at the rate $\Tilde{O}(\frac{1}{\sqrt{t}})$.
\end{theorem}
\begin{proof} 
The proof idea depends on bounding $\msf L(f_t) - \msf L(f^*)$ in terms of Bregman divergences of iterates $f_i$'s \tt{wrt} the mirror map $\Phi$. Note that we don't assume anything on the differentiability of the loss functional $\msf L$, thus $\partial_{f}\msf L$ needn't have a \tt{unique} subdifferential. Thus, at each iteration step $t$ of projected MDA in \eqnref{eq: constMD}, we use a subdifferential $\hat{g}_t \in \partial_{f_t}\msf L$. In the proof below, we would use $\partial_{f_t}\msf L$ for the subdifferential without loss of generality. 

Also, since $\Phi$ is proper and strongly convex, using \lemref{lemma: inversemirror} we know that $\partial_{(\cdot)} \Phi$ is invertible. Since $\cB$ is reflexive, convex conjugate, denoted as $\Phi^*$, provides the inverse operator from $\cB^*$ to $\cB$. 

Using convexity of $\msf L$  and \lemref{lemma: natdirection} we can write
    \begin{align}
        \msf L(f_t) - \msf L(f^*) &\le  \inner{f_t - f^*,\, \partial_{f_t}\msf L}_{\cB}\nonumber\\
    &\le \frac{1}{\eta} \inner{f_t - f^*,\, \partial_{f_t} \Phi - g_{t+1}}_{\cB} \label{eq:pmd1}\\
    & = \frac{1}{\eta}\inner{f_t - f^*,\, \partial_{f_t}\Phi - \partial_{\partial^{}_{g_{t+1}} \Phi^*} \Phi}_{\cB}\label{eq:pmd2}
\end{align}
In \eqnref{eq:pmd1}, we have used the gradient step in the dual space (see \eqnref{eq: constMD}). Since $\partial_{(\cdot)} \Phi^*: \cB^* \to \cB$ is the inverse operator to $\partial_{(\cdot)} \Phi$, we can write $g_{t+1 = \partial_{\paren{\partial_{g_{t+1}} \Phi^*}}} \Phi$. 

Now, using the identity on the Bregman divergence $\mf D_{\Phi}$ in \eqnref{eq: bregid}, we can write \vspace{-4mm}

\begin{align*}
 &\inner{f_t - f^*,\, \partial_{f_t}\Phi - \partial_{\partial^{}_{g_{t+1}} \Phi^*} \Phi}_{\cB}\\
    & = \mf D_{\Phi}(f_t,\partial^{}_{g_{t+1}} \Phi^*) + \mf D_{\Phi}(f^*,f_t) - \mf D_{\Phi}(f^*,\partial^{}_{g_{t+1}} \Phi^*)\\
    & \le \mf D_{\Phi}(f_t,\partial^{}_{g_{t+1}} \Phi^*) + \mf D_{\Phi}(f^*,f_t) - \mf D_{\Phi}(f^*,f_{t+1}) - \mf D_{\Phi}(f_{t+1},\partial^{}_{g_{t+1}} \Phi^*)\\
    & = \underbrace{\mf D_{\Phi}(f^*,f_t) - \mf D_{\Phi}(f^*,f_{t+1})}_{I} +  \underbrace{\mf D_{\Phi}(f_t,\partial^{}_{g_{t+1}} \Phi^*)  - \mf D_{\Phi}(f_{t+1},\partial^{}_{g_{t+1}} \Phi^*)}_{II}
    \end{align*}\vspace{-5mm}

Now, rest of the proof follows similar steps as in \cite{bubeck2015convex} (Theorem 4.2). 
Summing over $i = 1$ to $i = t$ gives the following bound on (I):
    \begin{align*}
        (I) & = \sum_{i = 1}^{t} \mf D_{\Phi}(f^*,f_i) - \mf D_{\Phi}(f^*,f_{i+1})\\
        & =  \mf D_{\Phi}(f^*,f_1) - \mf D_{\Phi}(f^*,f_{t+1})\\
        & \le R^2 
    \end{align*}
    To bound (II), we note the following:
    \allowdisplaybreaks
    \begin{align}
        (II) & =  \mf D_{\Phi}(f_t,\partial^{}_{g_{t+1}} \Phi^*)  - \mf D_{\Phi}(f_{t+1},\partial_{g_{t+1}} \Phi^*) \nonumber\\
        & = \Phi(f_t) - \Phi(f_{t+1}) - \inner{f_t - f_{t+1}, \partial_{\partial_{g_{t+1}} \Phi^*} \Phi}_{\cB} \label{eq: 21}\\
        & \le \inner{f_t - f_{t+1}, \partial_{f_t} \Phi}_{\cB} - \frac{\mu}{2}||f_t - f_{t+1}||^2_{\cB} - \inner{f_t - f_{t+1}, \partial_{\partial_{g_{t+1}} \Phi^*} \Phi}_{\cB}\label{eq: 22}\\
        & = \inner{f_t - f_{t+1}, \partial_{f_t} \Phi - \partial_{\partial_{g_{t+1}} \Phi^*} \Phi}_{\cB} - \frac{\mu}{2}||f_t - f_{t+1}||^2_{\cB}\nonumber\\
        & = \inner{f_t - f_{t+1}, \eta\cdot \partial_{f_t} \msf L}_{\cB} - \frac{\mu}{2}||f_t - f_{t+1}||^2_{\cB}\label{eq: 24}\\
        & \le \eta\cdot||f_t - f_{t+1}||_{\cB}\cdot ||\partial_{f_t} \msf L||_{B^*} - \frac{\mu}{2}||f_t - f_{t+1}||^2_{\cB}\label{eq: 25}\\
        & \le \eta L\cdot ||f_t - f_{t+1}||_{\cB} - \frac{\mu}{2}||f_t - f_{t+1}||^2_{\cB}\label{eq: 26}\\
        & \le \frac{(\eta L)^2}{2\mu} \label{eq: 27}
    \end{align}
In \eqnref{eq: 21}, we have used the definition of the Bregman divergence as shown in \eqnref{eq: BD}. We bound the difference $\Phi(f_t) - \Phi(f_{t+1})$ using the $\mu$-strong conexity of the mirror map $\Phi$ in \eqnref{eq: 22}. In \eqnref{eq: 24}, we use the mirror descent update of \eqnref{eq: constMD}. Now, it remains to bound $\inner{f_t - f_{t+1}, \eta\cdot \partial_{f_t} \msf L}_{\cB}$ in terms of $|||f_t - f_{t+1}|_\cB$ for which we use the Cauchy-Schwarz inequality, which is what we achieve in \eqnref{eq: 25}. In \eqnref{eq: 26}, we used the $L$-Lipshitzness of the loss functional $\msf L$. Finally, we use the following inequality $aq - bq^2 \le \frac{a^2}{4b}$ for any $q \in \reals$ to achieve the final bound in \eqnref{eq: 27}.

Now, adding for $t$ iterations we get
\begin{align*}
    \sum_{i=1}^t (\msf L(f_t) - \msf L(f^*)) & = \frac{1}{\eta}\paren{\mf D_{\Phi}(f^*,f_1) - \mf D_{\Phi}(f^*,f_{t+1}) + \frac{t(\eta L)^2}{2\mu}} \le \frac{R^2}{\eta} + \frac{\eta t L^2}{2\mu}
\end{align*}
 But using convexity 
 \begin{align*}
     \frac{1}{t} \sum_{i=1}^t \msf L(f_t) - \msf L(f^*) \le \msf L \paren{\frac{1}{t}\sum_{i=1}^t f_t} - \msf L(f^*) \le \frac{R^2}{t\eta} + \frac{\eta L^2}{2\mu}
 \end{align*}
 If apriori we set $\eta = \frac{R}{L}\sqrt{\frac{2\mu}{t}}$, then $\frac{R^2}{t\eta} + \frac{\eta L^2}{2\mu} = \frac{\sqrt{2}RL}{\sqrt{\mu t} }$, which gives the stated claim on the rate of convergence.
\end{proof}

\begin{remark} \thmref{thm: constopt} is stated without any assumption on the smoothness and differentiability of the loss functional $\msf L$. But with $\beta$-smoothness, one can hope to achieve a rate of $O(\frac{1}{t})$ with a slight variant of MDA by extending the idea of mirror prox as shown in \cite{bubeck2015convex} to the functional setting.

\end{remark}

\section{Example of MD on an RKBS}\label{sec: example}

In this section, we show a construction of a non-Hilbertian Banach space with a reproducing property. This space is spanned by eigenfunctions centered at chosen datapoints. Using these eigenfunctions, that are written in the form of kernel evaluations, we construct the unique kernel of the Banach space. Furthermore, with a standard $\ell_p$ type mirror map, we provide the explicit form of the mirror descent in \algoref{alg: md}.

Consider a locally compact Hausdorff set $\cX \subseteq \reals^d$. Consider a bivariate function $H: \cX \times \cX \to \reals\setminus\curlybracket{- \infty, + \infty}$. Fix a set of centers $C := \curlybracket{\bc_1,\bc_2,\ldots,\bc_k} \subset \cX$ such that $\curlybracket{H(\cdot,\bc_i)}_{i=1}^k$ 
is a set of linearly independent functions over the field $\reals$.

Now, for $p \neq 2$ and $k > 0$, we define the space of functions $\ell_p^k(C)$ as follows:
\begin{align}
  \ell_p^k(C) := \condcurlybracket{\sum_{i=1}^k \alpha_i\cdot H(\cdot,\bc_i)}{ \alpha \in \reals^k,\, \paren{\sum_{i=1}^k |\alpha_i|^p}^{\frac{1}{p}} < \infty}  \label{eq: RKBS}
\end{align}
The underlying norm is induced by the $\ell_p$ norm, i.e. for any $f_{\alpha} \in \ell_p^k$, $\norm{f_{\alpha}}_{\ell_p^k} = \paren{\sum_{i=1}^k |\alpha_i|^p}^{\frac{1}{p}}$. Note that due to the linear independence assumption, each function $f\in\ell^k_p$ has a unique representation in terms of the basis functions, whereby the norm is uniquely defined. \cite{nemirovski1983problem} studied a similar Banach space with finite length sequences of size $n$ with $p$-norm. Here, we differ by using the dot product of the vectors $\alpha$ with the evaluations of the bivariate map $H$.  

In this section, we study the real vector space $\paren{\ell_p^k(C), ||\cdot||_{\ell_p^k}}$. For ease of notation we would only write $\ell_p^k$ unless stated otherwise to signify a different choice of centers or bivariate function $H$. We will show that $\paren{\ell_p^k(C), ||\cdot||_{\ell_p^k}}$ is a reproducing kernel Banach  space.

First, we will show that the norm $||\cdot||_{\ell_p^k}$ is not induced by an inner product if $p \neq 2$. The key is to observe that any inner product that induces the norm $\norm{\cdot}_{\ell_p^k}$ would violate the parallelogram law. 

Let $\alpha = (1,1,0,\ldots,0)$ and $\alpha' = (1,-1,0,\ldots,0)$. Note that by parallelogram law,
\begin{align*}
  2||f_{\alpha}||_{\ell_p^k}^2 + 2||f_{\alpha}||_{\ell_p^k}^2 = ||f_{\alpha} + f_{\alpha'}||_{\ell_p^k}^2 + ||f_{\alpha} - f_{\alpha'}||_{\ell_p^k}^2  
  \implies 4\cdot 2^{2/p} = 8 \implies p = 2.
\end{align*}
Note that this space is isometrically isomorphic to the standard $\ell_p$ space of finite sequences with the $\norm{\cdot}_{\ell_p}$ norm. Thus, $\paren{\ell_p^k(C), ||\cdot||_{\ell_p^k}}$ is a complete normed vector space, aka a Banach space. 

Similarly, for $q = 1 + \frac{q}{p}$, we define the Banach space $\ell_q^k$ for a set linearly independent functions $\curlybracket{H(\bc_i,\cdot)}_{i=1}^k$ (over the field $\reals$) with the corresponding norm $\norm{\cdot}_{\ell_q^k}$ as follows:
\begin{align}
  \ell_q^k(C) := \condcurlybracket{\sum_{i=1}^k \alpha_i\cdot H(\bc_i,\cdot)}{ \alpha \in \reals^k,\, \paren{\sum_{i=1}^k |\alpha_i|^q}^{\frac{1}{q}} < \infty}  \label{eq: RKBSdual}
\end{align}
We would use $(\ell_q^k,\norm{\cdot}_{\ell_q^k})$ with the implicit understanding that space in \eqnref{eq: RKBSdual} is assumed.

\begin{remark}
    We assume that the span $\left\langle H(\bc_i,\cdot): i \in \bracket{k}\right\rangle$ to be dense in $\ell_q^k$ and span $\left\langle H(\cdot, \bc_i): i \in \bracket{k}\right\rangle$ to be dense in $\ell_p^k$. 
\end{remark}

In the following, we would show that $\paren{\ell_p^k, ||\cdot||_{\ell_p^k}}$ is endowed with the reproducing property (see \defref{def: RK}).

\begin{lemma}\label{lemma: showRKBS} For $1<p<\infty$, $\paren{\ell_p^k, ||\cdot||_{\ell_p^k}}$ is a reproducing kernel Banach space. 
\end{lemma}
\begin{proof} We need to show that the evaluation functional on the space $\ell_p^k$ is continuous, or there exists a positive constant $C_{\x}$ for any given ${\x} \in X$ such that 
\begin{align*}
  |\delta_{\x}(f_{\alpha})| = |f_{\alpha}({\x})| \le C_{\x}\cdot||f_{\alpha}||_{\ell_p^k}  
\end{align*}
for all $f_{\alpha} \in \ell_p^k$. It is easy to note for $C_{\x} = k\cdot \max_{i\in \bracket{k}} |H(\bc_i, {\x})|$ suffices for this as
\begin{align*}
  |f_{\alpha}(x)| = \bigg\lvert\sum_{i=1}^k \alpha_i\cdot H({\x},\bc_i)\bigg\rvert \le \max_{i\in \bracket{k}} |H({\x},\bc_i)| \cdot \sum_{i=1}^k |\alpha_i| \le k \cdot\max_{i\in \bracket{k}} |H({\x},\bc_i)| \cdot \paren{\sum_{i=1}^k |\alpha_i|^p}^{\frac{1}{p}}  
\end{align*}
\end{proof}

\paragraph{Dual space:} Now, we consider the dual space of $\ell_p^k$, denoted as ${\ell_p^k}^*$. We would show some interesting properties of this dual space. First, note that by definition, the dual norm is
\begin{equation}
  ||\msf F||_{{\ell_p^k}^*} = \sup_{||f||_{\ell_p^k} = 1} |\msf F(f)|  \label{eq: dualnorm}
\end{equation}

We would show that ${\ell_p^k}^*$ is isometrically isomorphic to $\ell_q^k$ where $\frac{1}{p} + \frac{1}{q} = 1$.

\begin{lemma}\label{thm: dual} Any bounded linear functional $\msf F \in \ell_p^{k*}$ can be uniquely represented as
$$\msf F(f_{\alpha}) := \sum_{i=1}^k \alpha_i \beta_i$$
for all $f_{\alpha} \in \ell_p^k$, where $\sum_{i=1}^k\beta_i\cdot H(\cdot,\bc_i) \in \ell_q^k$. Moreover, any function $g_\beta \in \ell_q^k$ defines a linear functional $\msf F$ in ${\ell_p^k}^*$ with dual norm
$$||F||_{{\ell_p^k}^*} = ||g_\beta||_{\ell_q^k} = \paren{\sum_{i=1}^k |\beta_i|^q}^{\frac{1}{q}}$$
\end{lemma}
\begin{proof}
    For this result, we assume\footnote{Case where $p = 1$ can be similarly handled.} that $p > 1$. Consider an arbitrary function $f_{\alpha} \in \ell_p^k$. Note that we can write
    \begin{align*}
        \msf F(f_{\alpha}) = \sum_{i=1}^k \alpha_i\cdot \msf F(H(\cdot,\bc_i))\hspace{10mm} (\textit{by linearity})
    \end{align*}
    Define $\beta_i$ as $\msf F(H(\cdot,\bc_i))$ for all $i \in \bracket{k}$. Now, we wish to show that $g_\beta := \sum_{i=1}^k \beta_i\cdot H(\bc_i,\cdot) \in \ell_q^k$, i.e. $\norm{g_\beta}_{\ell_q^k}$ is bounded. \\
    
    Consider a choice of $\alpha'$ such that
    $$ \alpha_i' = |\beta_i|^{\frac{q}{p}} \sgn{\beta_i}$$
    Thus, the norm of $f_{\alpha'}$ is 
    $$||f_{\alpha'}||_{\ell_p^k} = \paren{\sum_{i=1}^k |\beta_i|^{q}}^{\frac{1}{p}}.$$
    On the other hand,
    $$|\msf F(f_{\alpha'})| = \sum_{i=1}^k |\beta_i|^{\frac{q}{p}+1} $$
    But by definition
    \begin{align}
        &|\msf F(f_{\alpha'})| \le ||\msf F||_{{\ell_p^k}^*}\cdot ||f_{\alpha'}||_{\ell_p^k} \nonumber\\
        \implies & \sum_{i=1}^k |\beta_i|^{\frac{q}{p}+1} \le ||\msf F||_{{\ell_p^k}^*} \cdot \paren{\sum_{i=1}^k |\beta_i|^{q}}^{\frac{1}{p}}\nonumber\\
        \implies & \sum_{i=1}^k |\beta_i|^{q} \le ||\msf F||_{{\ell_p^k}^*} \cdot \paren{\sum_{i=1}^k |\beta_i|^{q}}^{\frac{1}{p}}\hspace{5mm} \paren{\textit{since } \frac{q}{p} + 1 = q}\nonumber\\
        \implies & \paren{\sum_{i=1}^k |\beta_i|^q}^{\frac{1}{q}} \le ||\msf F||_{{\ell_p^k}^*} \label{eq: uppbound}
    \end{align}
    Since $\msf F$ is bounded the inequality above implies that $g_\beta \in \ell_q^k$. Now, assuming $f_{\alpha''} \in \ell_p^k$ and $g_\beta \in \ell_q^k$ such that $\norm{f_{\alpha''}}_{\ell_p^k} = 1$, using H\"{o}lder's inequality we note that
    \begin{align}
        |\msf F(f_{\alpha''})| = \bigg\lvert \sum_{i=1}^k \alpha''_i\beta_i \bigg\rvert \le ||f_{\alpha''}||_{\ell_p^k}\cdot ||g_\beta||_{\ell_q^k} \label{eq: boundF}
    \end{align}
    Plugging the bound of \eqnref{eq: boundF} in \eqnref{eq: dualnorm}, we get
    \begin{align*}
        ||\msf F||_{{\ell_p^k}^*} \le  ||g_\beta||_{\ell_q^k} 
    \end{align*}
    But \eqnref{eq: uppbound} implies that $||\msf F||_{{\ell_p^k}^*} =  ||g_\beta||_{\ell_q^k}$. 
    
    Showing that any element in $\ell_q^k$ defines a linear functional is straightforward where we consider a linear functional $\msf F$ that maps $H(\bc_i,\cdot)$ for all $i \in \bracket{k}$ to the corresponding parameters. This completes the proof.
\end{proof}

Thus, we can treat $\paren{\ell_q^k,\norm{\cdot}_{\ell_q^k}}$ as the dual of $\paren{\ell_q^k,\norm{\cdot}_{\ell_q^k}}$ where any element in $\ell_q^k$ is \tt{identified} as a linear transformation in ${\ell_p^k}^*$.

\paragraph{Reproducing kernel of $\ell_{p}^k$:} In \lemref{lemma: showRKBS}, we showed that $\paren{\ell_p^k, \norm{\cdot}_{\ell_p^k}}$ has the reproducing property. A question remains if we could find an explicit form for a kernel $K$ for the constructed RKBS.
In the following discussion, we show explicit kernels for the primal and dual spaces $\paren{\ell_q^k,\norm{\cdot}_{\ell_q^k}}$, and $\paren{\ell_q^k,\norm{\cdot}_{\ell_q^k}}$ respectively, which are both RKBSs. Depending on how we treat $\ell_q^k$ (dual or adjoint) we can show different kernels for the pair. As shown in Theorem 1~\citep{zhang09b}, if $\ell_q^k$ is treated as a dual then there is a unique kernel, otherwise as  an adjoint one can show multiple kernels. In the following, we will show the explicit constructions.

\textit{Kernel for adjoint Banach spaces}: Here, we study $\ell_p^k$ and $\ell_q^k$ as  Banach spaces \textit{adjoint} to each other as per \defref{def: RK}. We study a choice of bilinear map and show that how $H$ turns out to be the kernel in this setting. Consider the following bilinear form $\inner{\cdot,\cdot}_{\cB}$ on $\ell_p^k \times \ell_q^k$ defined as  
\begin{align}
  \inner{f_{\alpha}, \bg{\beta}}_{\cB} := \sum_{i,j} \alpha_i\beta_j K(\bc_j,\bc_i)  
\end{align}
It is easy to check its continuity. 
Furthermore,
\begin{gather*}
    \inner{f_{\alpha}, H({\x},\cdot)}_{\cB} = \sum_{i = 1}^k \alpha_iH({\x},\bc_i) = f_{\alpha}({\x}),\quad {\x} \in \cX\\
    \inner{H(\cdot,\by), \bg{\beta}}_{\cB} = \sum_{i = 1}^k \beta_iH(\bc_i,\by) = \sum_{i = 1}^k \beta_iH(\bc_i,\by) = \bg{\beta}(\by), \quad \by \in \cX
\end{gather*}
Thus, $H$ forms a reproducing kernel for $\ell_p^k$ and $\ell_q^k$ where we consider them as adjoint under a specific choice of bilinear form.

In the following, we study the case when a canonical bilinear form is induced by the dual space.

\textit{Canonical bilinear form and its unique kernel}: Here, we consider the bilinear form induced by action of the dual space ${\ell_p^k}^*$ on the primal space $\ell_p^k$. We would use this canonical action to define a bilinear map on $\ell_p^k \times \ell_q^k$. Any function $g_\beta \in \ell_q^k$ is identified as an element in $\msf F_\beta \in {\ell_p^k}^*$ as shown in \thmref{thm: dual}. Thus, we write
\begin{align*}
    \inner{f_{\alpha},\bg{\beta}}_{\ell_p^k \times \ell_q^k} = \inner{f_{\alpha}, \msf F_{\beta}}_{\ell_p^k \times {\ell_p^k}^*} = \msf F_{\beta}(f_{\alpha}) = \sum_{i = 1}^k \alpha_i\beta_i
\end{align*}
Note, that $H$ can't be the underlying kernel as 
\begin{equation*}
    \inner{f_{\alpha}, H(\bc_j,\cdot)}_{\ell_p^k \times \ell_q^k} = \alpha_j,
\end{equation*}
whereas the evaluation should be $\sum_{i=1}^k \alpha_iH(\bc_j,\bc_i)$. Now, consider a bivariate function $K: \cX \times \cX \to \reals$ that is defined for ${\x} \in \cX$ as 
\begin{align*}
    K(\cdot,{\x}) := \sum_{i =1}^k H(\bc_i,{\x})H(\cdot,\bc_i)\\
    K({\x},\cdot) := \sum_{i =1}^k H({\x},\bc_i)H(\bc_i, \cdot) 
\end{align*}

It is easy to check that for all ${\x} \in \cX$ we have $K(\cdot,{\x}) \in \ell_p^k$ and $K({\x},\cdot) \in \ell_q^k$. Furthermore, we note that
\begin{align*}
    \inner{f_{\alpha}, K({\x},\cdot)}_{\ell_p^k \times \ell_q^k} = \sum_{i=1}^k \alpha_iH({\x}, \bc_i) = f_{\alpha}({\x})\\
    \inner{K(\cdot, {\x}), \bg{\beta}}_{\ell_p^k \times \ell_q^k} =  \sum_{i=1}^k H(\bc_i,{\x})\beta_i = \bg{\beta}({\x})
\end{align*}


\begin{remark}
The construction of the RKBS as shown in \eqnref{eq: RKBS} can be extended to an arbitrary countable set of centers. This has been studied in \cite{xu2019generalized}. 
They consider a locally compact Hausdorff space $\Omega$ with measure $\mu$  in $\reals^d$ and eigenfunctions $\phi_n \in L_0(\Omega)$ for all $n \in \mathbb{N}$ and show that the following space
\begin{equation}
    B^p_{K}(\Omega) := \condcurlybracket{f := \sum_{i \in \nats} a_i\phi_i}{(a_i \in \nats) \in \ell_p}
\end{equation}
is an RKBS with the norm $||\cdot||_{B^p_{K}} = ||\cdot||_{\ell_p}$,
where $K \in L_0(\Omega \times \Omega)$ is a generalized Mercer kernel as defined in \cite{xu2019generalized}. Note that we don't assume any condition on the bivariate function $H$ other than the linear independence conditions over the fixed centers.   
\end{remark}

\paragraph{Mirror map} To instantiate the MDA, we consider a natural choice of mirror map. Consider the functional $\Phi_p: \ell_p^k \to \reals$ defined as follows:
\begin{align*}
    \Phi_p(f_{\alpha}) := \frac{1}{2}\cdot\norm{f_\alpha}_{\ell_p^k}^2, \quad \forall\, f_\alpha \in \ell_p^k 
\end{align*}
Note that the Fr\'{e}chet differential of $\Phi_p$ at $f_\alpha$ with increment $f_{\alpha'}$ is 
\begin{align*}
    \partial_{f_\alpha} \Phi_p(f_{\alpha'}) = \sum_{i = 1}^k \alpha_i'\cdot \frac{\sgn {\alpha_i}|\alpha_i|^{p-1}}{||\alpha||_p^{p-2}}
\end{align*}
Since ${\ell_p^k}^*$ is isometrically isomorphic to $\ell_q^k$, using \thmref{thm: dual} we have the following correspondence
\begin{align}
    \partial_{f_\alpha} \Phi_p \equiv g_{\beta} = \sum_{i=1}^k \beta_i\cdot H(\bc_i,\cdot),\,\text{ where } \beta = \frac{\sgn {\alpha}|\alpha|^{p-1}}{||\alpha||_p^{p-2}} \label{map: duallp}
\end{align}
Now, refer to the map
\begin{align}
    \partial_{(\cdot)} \Phi_p: \ell_p^k \to \ell_p^k,\, f\mapsto \partial_f \Phi_p
\end{align}
As a direct consequence of \lemref{lemma: inversemirror}, it is straight-forward to show that $ \partial_{(\cdot)} \Phi_p$ is invertible. Furthermore, the inverse has some desired properties as well.
\begin{corollary}\label{cor: lp} Fix any $p \in [1,\infty)$ and consider the Banach space $\ell_p^k$ as shown in \eqnref{eq: RKBS}. Then, the inverse of the G\^{a}teaux derivative of the map $\frac{1}{2}||\cdot||_p^2$ on $\ell_p^k$ is the G\^{a}teaux derivative of $\frac{1}{2}||\cdot||_q^2$ on $\ell_q^k$, where $\frac{1}{p} + \frac{1}{q} = 1$.
\begin{proof}
    Note that $\frac{1}{2}||\cdot||_p^2$ is a proper strictly convex, G\^{a}teaux differentiable map on $\ell_p^k$.
    Since the Fenchel conjugate of $\frac{1}{2}||\cdot||_p^2$ is $\frac{1}{2}||\cdot||_q^2$ (cf Example 3.27~\cite{Boyd2004}), the claim of the corollary follows immediately using \lemref{lemma: inversemirror}.
\end{proof} 
\end{corollary}

In \algoref{alg: md}, we provide explicit updates for MDA of \eqnref{eq: MD} for the optimization problem of \eqnref{eq: optPS} over $\ell_p^k$ for squared-error loss. The algorithm maps functions in the primal space $\ell_p^k$ to the dual space $\ell_q^k$ via a mirror map as shown in \eqnref{map: duallp}.
Below, we provide the justification for the update steps in \algoref{alg: md}. 

First, we note that the linear combination of the inverse of a function $g_{\beta}$ wrt $(\partial_{(\cdot)} \Phi_p)^{-1}$, say $\alpha$, can be computed in terms of $\beta$.
\begin{algorithm}[t]
\caption{Mirror descent on $\ell_p^k$}
\label{alg: md}
\KwData{Given a training set $\curlybracket{(\x_i,y_i)}_{i=1}^n$, model class $\ell_p^k$, similarity kernel $\widehat{H}$, initial parameters $\alpha_0$, $\beta_0$}
\KwResult{Parameters $\widehat{\alpha}$}
\nl Initialize $t = 0$\;
\While{$t \le T$}{
 \nl   $\beta^{(t)} \gets \beta^{(t-1)} - \eta\cdot(\widehat{H}^{\top}(\widehat{H}\alpha^{(t-1)} - Y))$\label{step1}\;
 \nl   $\alpha^{(t)} \gets \frac{\sgn {\beta^{(t)}} |\beta^{(t)}|^{q-1}}{\norm{\beta^{(t)}}_q^{q-1}}$\label{step2}\;
  \nl  t++\;
}
\nl return $\widehat{\alpha} = \alpha_T$\;
\end{algorithm}
\begin{lemma}\label{lemma: explicit}
    We have $\round{\partial_{(\cdot)}\Phi_p}^{-1} g_\beta = f_\alpha$
    where $\alpha_i = \sgn{\beta_i}\frac{\abs{\beta_i}^{q-1}}{\norm{\beta}_q^{q-1}}$ for all $i \in \bracket{k}$ and $\frac{1}{p} + \frac{1}{q} = 1$.
\end{lemma}
\begin{proof}
    Using \corref{cor: lp}, it is clear that $\round{\partial_{(\cdot)}\Phi_p}^{-1} = \partial_{(\cdot)}\Phi_q$. Now, since $\Phi_q: \ell_q^k \to \reals$, thus we note that 
    \begin{gather}
     \partial_{(\cdot)} \Phi_q: \ell_q^k \to \ell_p^k \nonumber\\
     g_\beta \mapsto \sum_{i=1}^k \sgn{\beta_i}\frac{\abs{\beta_i}^{q-1}}{\norm{\beta}_q^{q-1}}\cdot H(\cdot,\bc_i) \nonumber
    \end{gather}
which is immediate using \eqnref{map: duallp}.
\end{proof} 
We consider the following notations to set up the algorithm.

\tt{Notations}: Coefficients of a mirror descent update at iteration step $t$ in MDA (see \eqnref{eq: MD}) are denoted as $\beta^{(t)}$ for the function $g_{\beta^{(t)}}$ in $\ell_q^k$ (similarly $\alpha^{(t)}$ for $f_{\alpha^{(t)}} \in \ell_p^k$). $\alpha$, $\beta$ and $Y := (y_1, y_2,\ldots,y_n)^T$ are treated as column vectors in $\reals^k$. 
We define a similarity matrix $\widehat{H}$ of dimension $(n \times k)$ where $\widehat{H}_{ij} := H(\x_i,\bc_j)$. For the set of centers $C = \curlybracket{\bc_1,\bc_2,\ldots,\bc_n}$, we also use $\bar{\bc}$ to denote a vector, e.g. $H(\cdot,\bar{\bc})$ denotes a row vector $(H(\cdot,\bc_1),H(\cdot,\bc_2),\ldots,H(\cdot,\bc_k))$. Similarly, for a training set $\curlybracket{\x_1,\x_2,\ldots,\x_n}$ we use $\bar{\x}$ to denote row vector $H(\bar{\x},\cdot) := (H(\x_1,\cdot),H(\x_2,\cdot),\ldots,H(\x_n,\cdot))$.

Using the notations above, for clarity we rewrite the update steps of \eqnref{eq: MD} here for squared error loss:
\begin{subequations}\label{eq: pnorm}
\begin{align}
    g_{\beta^{(t)}} &\gets g_{\beta^{(t-1)}} - 2\eta \sum_{i=1}^n (f_{\alpha^{(t-1)}}(\x_i) - y_i)\cdot K(\x_i,\cdot)  \\
    f_{\alpha^{(t)}} &\gets \paren{\partial \Phi_p}^{-1} (g_{\beta^{(t)}})
\end{align}
\end{subequations}
Explicit updates for the first equation~\lineref{step1} of \algoref{alg: md} can be derived as follows:
\allowdisplaybreaks
\begin{align*}
    \sum_{i=1}^k \beta^{(t)}_i\cdot H(\bc_i,\cdot) &= \sum_{i=1}^k \beta^{(t-1)}_i\cdot H(\bc_i,\cdot) -2\eta \sum_{i=1}^n  \bigparen{\sum_{j =1}^k (f_{\alpha^{(t-1)}}(\x_i) - y_i 
    )H({\x_i},\bc_j)H(\bc_j, \cdot)} \\
    &= \sum_{i=1}^k \beta^{(t-1)}_i\cdot H(\bc_i,\cdot) -2\eta \sum_{i=1}^n \sum_{j =1}^k \bigparen{\sum_{m=1}^k \alpha^{(t-1)}_m\cdot H(\x_i,\bc_m) - y_i}\cdot H({\x_i},\bc_j)H(\bc_j, \cdot) \\
    &= \sum_{i=1}^k \beta^{(t-1)}_i\cdot H(\bc_i,\cdot) -2\eta \sum_{j =1}^k \sum_{i=1}^n \bigparen{\sum_{m=1}^k \alpha^{(t-1)}_m\cdot H(\x_i,\bc_m) - y_i}\cdot H({\x_i},\bc_j)H(\bc_j, \cdot)\\
    &= \sum_{i=1}^k \beta^{(t-1)}_i\cdot H(\bc_i,\cdot) -2\eta \sum_{j =1}^k\paren{ \sum_{i=1}^n \bigparen{\sum_{m=1}^k \alpha^{(t-1)}_m\cdot H(\x_i,\bc_m) - y_i}\cdot H({\x_i},\bc_j)}H(\bc_j, \cdot)\\
    &= \sum_{i=1}^k \beta^{(t-1)}_i\cdot H(\bc_i,\cdot) -2\eta \sum_{j =1}^k\paren{ \sum_{i=1}^n \bigparen{H(\x_i,\bar{\bc})\alpha^{(t-1)} - y_i}\cdot H({\x_i},\bc_j)}H(\bc_j, \cdot)\\
    &= \sum_{i=1}^k \beta^{(t-1)}_i\cdot H(\bc_i,\cdot) -2\eta \sum_{j =1}^k\paren{ \sum_{i=1}^n H(\x_i,\bar{\bc})\alpha^{(t-1)}H({\x_i},\bc_j) - H({\bar{\x}},\bc_j)^T Y}H(\bc_j, \cdot)\\
    &= \sum_{i=1}^k \beta^{(t-1)}_i\cdot H(\bc_i,\cdot) -2\eta \sum_{j =1}^k\paren{ H({\bar{\x}},\bc_j)^T H(\bar{\x},\bar{\bc})\alpha^{(t-1)} - H({\bar{\x}},\bc_j)^T Y}H(\bc_j, \cdot)\\
    &= \sum_{i=1}^k \beta^{(t-1)}_i\cdot H(\bc_i,\cdot) -2\eta \sum_{j =1}^k H({\bar{\x}},\bc_j)^T\paren{ H(\bar{\x},\bar{\bc})\alpha^{(t-1)} - Y}H(\bc_j, \cdot)\\
    &= \sum_{i=1}^k \paren{\beta^{(t-1)}_i -2\eta H({\bar{\x}},\bc_i)^T\paren{ H(\bar{\x},\bar{\bc})\alpha^{(t-1)} - Y}}H(\bc_i, \cdot)\\
\end{align*}
This implies that $\beta^{(t)} = \beta^{(t-1)} - 2\eta \widehat{H}^T\paren{\widehat{H}\alpha^{(t-1)} - Y}$.
Using \lemref{lemma: explicit}, \lineref{step2} follows:

\begin{align*}
    \alpha^{(t)} \gets \frac{\sgn {\beta^{(t)}} |\beta^{(t)}|^{q-1}}{\norm{\beta^{(t)}}_q^{q-1}}
\end{align*}
Thus, the full procedure in \algoref{alg: md} can be written down without retaining evaluations $H(\bc_i,\cdot)$'s for linear transformations in the dual space $\ell_q^k$ (similarly for functions in the primal space $\ell_p^k$).

In \algoref{alg: md}, the updates are based on a canonical view of kernel. In the case of adjoint based kernel, the updates for $\beta^{(t)}$ would be sligtly different:
\begin{align*}
    \beta^{(t)} \gets \beta^{(t-1)} - \eta\cdot(\widehat{H}\alpha^{(t-1)} - Y)
\end{align*}

\paragraph{Convergence of \algoref{alg: md}:}
We utilize the square loss functional, which is \(L\)-Lipschitz with respect to the \(p\)-norm of the RKBS, assuming that \(\alpha\) has a bounded norm. According to \cite{JMLR:v13:kakade12a} (see Lemma 9), the strong convexity parameter of the mirror map \(\frac{1}{2}\|\cdot\|_p^2\) is \((p-1)\) for \(p \in (1,2)\). For \(p > 2\), the corresponding squared \(p\)-norm is not strongly convex. Consequently, the convergence of Algorithm \ref{alg: md} is theoretically guaranteed at the rate shown in Theorem \ref{thm: constopt} under \(p\)-norm RKBSs for \(p \in (1,2)\). We validate this convergence for various \(p\) values within this range for step functions, as demonstrated in Figure \ref{fig:example_image}.

\begin{figure}[t]
    \centering
    \includegraphics[width=\linewidth]{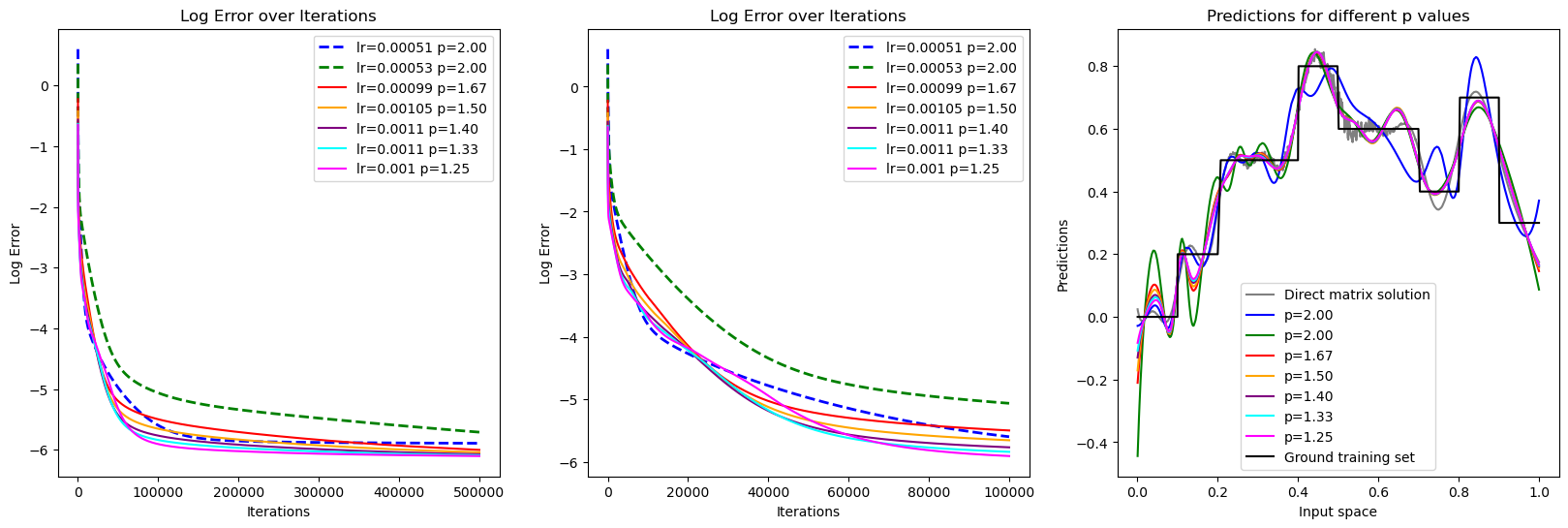}
    \caption{Results from numerical experiments using the mirror descent algorithm (\algoref{alg: md}) for varying \( p \) values in a one-dimensional space. We apply squared error loss on 800 training points with 25 centers using the Locally Adaptive-Bandwidths (LAB) RBF kernel. The bivariate function \( H \) is defined by \( \exp \left( -\frac{\| \theta_i \odot (x - \bc_i) \|_2^2}{2} \right) \), with \(\theta_i\) optimized via gradient descent. (\textbf{Leftmost Plot}): Logarithm of training error versus iterations. 
(\textbf{Middle Plot}): Zoomed-in log error versus iterations up to 100,000 steps.
(\textbf{Rightmost Plot}): Predictions of learned kernel classifiers compared with the direct matrix solution (gray curve) on the training set.}
    \label{fig:example_image}
\end{figure}
In the following, we discuss our numerical experiments. 

\paragraph{Numerical Experiments:} We present experimental results that validate our theoretical findings. The dynamics of convergence of Algorithm \ref{alg: md} for different values of \( p \) are plotted, for an optimization problem in a one-dimensional input space. Specifically, we visualize the impact of different \( p \) values on the approximation performance and training error.

To construct the space \( \ell_p^k \), we employ the Locally Adaptive-Bandwidths (LAB) RBF kernels as introduced in \cite{he2024learninganalysiskernelridgeless}. Using this kernel, the \( p \)-norm RKBS \( \ell_p^k \) is constructed using a bivariate function \( H \), defined as
$$
H(x, \bc_i) = \exp \left( -\frac{\| \theta_i \odot (x - \bc_i) \|_2^2}{2} \right), \quad \forall \bc_i \in C,
$$
where \( \odot \) denotes the Hadamard (element-wise) product, \( \| \cdot \|_2 \) is the \( L^2 \)-norm, and \( \theta_i \in \mathbb{R} \) is a center-dependent bandwidth vector. Note that \( H \) is asymmetric, which implies that the kernel \( K \) for the space \( \ell_p^k \) is also asymmetric. The bandwidths \( \theta_i \) are computed by performing 10 gradient steps on the loss function \( \mathsf{L}(\alpha, \boldsymbol{\theta}) \), where \( \boldsymbol{\theta} = [\theta_1, \theta_2, \ldots, \theta_k]^\top \), according to
$$
\boldsymbol{\theta}_t \leftarrow \boldsymbol{\theta}_{t-1} - \eta \partial_{\boldsymbol{\theta}_{t-1}} \mathsf{L}(\alpha, \boldsymbol{\theta}),
$$
where the learning rate \( \eta \) differs from the learning rate used in Algorithm \ref{alg: md}.

\tt{Experimental Setup:} The input space is fixed as \( \mathcal{X} = \mathbb{R} \). Using the computed \( \boldsymbol{\theta} \) from the gradient update, we run Algorithm \ref{alg: md} to learn a step function in one dimension over the interval \([0,1]\), as depicted in Figure \ref{fig:example_image} (black curve in the rightmost plot). The loss functional used is the squared error. Our training set consists of 800 points, with 25 centers randomly selected from this set, yielding \( k = 25 \) for constructing the space \( \ell_p^k \). We explore different values of \( p \) in the range \( (1,2] \), specifically \( \{2, 1.67, 1.5, 1.4, 1.33, 1.25\} \).

We present two types of plots: (1) the logarithm of the training error versus the number of iterations (see the leftmost and middle plots in Figure \ref{fig:example_image}), and (2) the predictions of the learned kernel classifiers \( \sum_{i=1}^{25} \alpha_i H(\cdot, \bc_i) \) on the training set (see the rightmost plot in Figure \ref{fig:example_image}).

The first plot illustrates the logarithm of the training error over 500,000 iterations for various \( p \) values and learning rates. To provide a clearer view of convergence trends, a zoomed-in plot of the log error versus iterations for up to 100,000 steps is also included. A general trend observed across the plots indicates that as \( p \) decreases, the approximation quality improves and the training error decreases. This effect is also visible in the predictions on the training points in the rightmost plot of Figure \ref{fig:example_image}. This improvement can be attributed to the expansion of the space as \( p \) decreases, as \( \|\alpha\|_p \) is an increasing function of \( p \). For a fixed norm \( D > 0 \), the space \( \ell_p^k \) grows larger with decreasing \( p \), allowing for richer approximations within the space.

\tt{Comparison with Direct Matrix Solution:} We also compare the predictions obtained using the solution from NumPy (depicted by the gray curve in the rightmost plot of Figure \ref{fig:example_image}). In this comparison, the NumPy solution solves the system of linear equations \( K_g \cdot \alpha = Y \), where \( K_g \) is the similarity matrix computed over the (training, centers) pairs using a Gaussian kernel, specifically
$$
K_g(x_i, c_j) = \exp\left(-\frac{\|x_i - c_j\|_2^2}{\sigma^2}\right),
$$
for training point and center pair \( (x_i, c_j) \). Here, \( K_g \) is a similarity matrix of dimension \( 800 \times 25 \). The prediction plot uses the best bandwidth \( \sigma^2 \) for the Gaussian kernel, chosen to yield optimal performance across different trials.




\acks{
\noindent We acknowledge support from the National Science Foundation (NSF) and the Simons Foundation for the Collaboration on the Theoretical Foundations of Deep Learning through awards DMS-2031883 and \#814639,  the  TILOS institute (NSF CCF-2112665), and the Office of Naval Research (N8644-NV-ONR). This work was started when P.P. was a postdoctoral fellow at the Halıcıoğlu Data Science Insititute at UC San Diego.
}
\newpage

\appendix
\section{Technical Proofs}\label{app: technical}

\subsection{Mirror maps: Proof of \lemref{lemma: inversemirror}}
First, we state a useful lemma on the first-order optimality condition for optimization in functional analysis.
\begin{lemma}[{\cite[Chapter 8, Lemma 1]{luen}}]\label{lemma: firstorder}
    Let $\msf F$ be a Fr\'{e}chet differentiable convex functional on a real
normed space $\cB$. Let $\cC'$ be a convex cone in $\cB$. A necessary and sufficient condition that $f_0 \in \cP$  minimizes $\msf F$ over $\cC'$ is that
\begin{align*}
    \partial_{f_0}\msf F (f) &\ge 0\qquad \forall\, f \in \cC',\\
    \partial_{f_0}\msf F (f_0) &= 0.
\end{align*}
\end{lemma}
We are interested in the case where the convex cone of interest is the space $\cB$ itself. The lemma above is stated for any real normed space $\cB$. Given that we are interested in Banach spaces, we could state the following more precise statement on the minimum of the functional $\msf F$.
\begin{proposition}[{\cite[Section 7.4, Theorem 1]{luen}}]
Let the real-valued functional \( \msf F \) have a G\^ateaux differential on a vector space \( \cB \). A necessary condition for \( \msf F \) to have an extremum at \( f_0 \in \cB \) is that \( \partial_{f_0} \msf F(h) = 0 \) for all \( h \in \cB \).
\end{proposition}

\begin{proof}{ of \lemref{lemma: inversemirror}:} It is straightforward to see that the strict convexity and G\^{a}teaux differentiability of \(\msf F\) imply its injectivity. However, we present an alternative proof that introduces notations used to demonstrate surjectivity. To establish the injectivity of the operator \(\partial_{(\cdot)} \msf F\), we consider the convex conjugate of \(\msf F\), denoted \(\msf F^*\). Note that \(\msf F^*: \cB^* \to \mathbb{R}\), and
    \begin{align}
        \forall \bg{}\in\cB^*,\, \msf F^*(\bg{}) := - \inf_{f \in \cB} (-\dist{f}{\bg{}} + \msf F(f)) \label{eq: conjugate}
    \end{align}
    We denote $\hat{\msf F}_{\bg{}}(f) := -\dist{f}{\bg{}} + \msf F(f)$.
    Note that since $\msf F$ is strictly convex, thus $\hat{\msf F}_{\bg{}}$ is strictly convex in $f$. Thus, the objective $\hat{\msf F}_{\bg{}}(f)$ can only have a \tt{unique} minimizer.
    
    Assume that there exists $\hat{f} \in \cB$ such that $\bg{} := \partial_{\hat{f}} \msf F$. Using the first-order optimality condition in \lemref{lemma: firstorder}, $\hat{f}$ is the minimizer of $\hat{\msf F}_{\bg{}}$, but then it is the unique minimizer. This implies that there doesn't exist $f' \neq \hat{f}$ such that $\partial_{f'} \msf F = \partial_{\hat{f}} \msf F$. This completes the proof of the injectivity of the operator $\partial_{(\cdot)} \msf F$.
    
    Now, we would establish the surjectivity of the operator $\partial_{(\cdot)} \msf F$ over $\cB^*$. Using Theorem 2.3.3~\cite{Zlinescu2002ConvexAI}, $\msf F^{**} = \msf F$, i.e.
    \begin{align}
        \forall f' \in \cB,\,\msf F(f') := - \inf_{\hat{\bg{}} \in \cB^*} (-\dist{f'}{\hat{\bg{}}} + \msf F^*(\hat{\bg{}})) \label{eq: dualconjugate}
    \end{align}
    Consider a linear transformation $\bg{}^* \in \cB^*$. Assume that $\hat{\msf F}_{\bg{}^*}$ is minimized at $f_0 \in \cB$. Thus,
    \begin{align}
        \msf F^*(\bg{}^*) = \dist{f_0}{\bg{}^*} - \msf F(f_0) \label{eq: l1}
    \end{align}
    Using \textbf{Young-Fenchel inequality} (Theorem 2.3.1 \cite{Zlinescu2002ConvexAI}), we note that for all $f \in \cB$,
    \begin{align}
        \msf F(f) + \msf F^*(\bg{}^*) \ge \dist{f}{\bg{}^*} \label{eq: l2}
    \end{align}
    Subtracting \eqnref{eq: l1} and \eqnref{eq: l2}, we observe that for all $f \in \cB$,
    \begin{equation*}
        \msf F(f) - \msf F(f_0) \ge \dist{f - f_0}{\bg{}^*}
    \end{equation*}
    This implies that $\bg{}^* \in \partial_{f_0} \msf F$, i.e $\bg{}^*$ is in the subdifferential set of $\partial_{f_0} \msf F$. Since $\msf F$ is G\^{a}teaux differentiable at $f_0$ it must be the case that $\bg{}^* = \partial_{f_0} \msf F$.

    Given that $\bg{}^*$ was picked arbitrarily, we have shown that for any $\bg{}^* \in \cB^*$ there exists $f_0 \in \cB$ such that $\bg{}^* = \partial_{f_0} \msf F$. This completes the proof of the surjectivity of the operator $\partial_{(\cdot)} \msf F$ over $\cB^*$.
\end{proof}
Using \eqnref{eq: l1} and \eqnref{eq: l2}, we note that the subjectivity of the mirror map $\Phi$ does not necessarily require it is G\^ateaux  differentiable.
\vspace{-2mm}

\subsection{Existence of a smooth and strongly convex functional: Proof of \lemref{lemma: existence}}\label{app: smooth-convex}

In this Appendix, we show that for a functional $\msf F: \cC \to \real$ over a Banach space to be both $\mu$-strongly convex and $\gamma$-smooth for $\mu >0 $ and $\gamma < \infty$, it must be isomorphic to a Hilbert space.

The key idea of the proof is to consider the points of second differentiability of a functional $\msf F$. Essentially, we wish to look at functions $f \in \cB$ where a Taylor approximation is possible. If we show that there exists such a function then one can achieve equivalence of the Banach space $||\cdot||_{\cB}$ to a Hilbert space norm. 

We show that indeed such a point exists where a Taylor expansion to the second order is possible if the functional $\msf F$ is continuous and convex. Consider the following lemma: 
\begin{lemma}\label{lemma: taylor}
    Consider \( \cB \) be a separable Banach space, and let \( \msf F \) be a continuous convex function on \( \cB \). Then, there exists an $f \in \cB$, $g^* \in \partial_f \msf F$, and a bounded symmetric linear operator \( \msf M: \cB \to \cB^* \) such that for all $f' \in \cB$ and small scalar $\lambda > 0$, we can expand \(\msf F\) as follows
    \begin{equation}
\msf F(f + \lambda f') = \msf F(f) + \lambda \langle f',g^* \rangle_\cB + \frac{\lambda^2}{2} \langle f',\msf M(f') \rangle_\cB + o(\lambda^2) \quad (\lambda \to 0). \label{eq: taylor}
\end{equation}
\end{lemma}
\begin{proof}
    Using Theorem 4.1~\cite{Borwein1994SecondOD}, it is straight-forward that the set of functions $f \in \cB$ where \eqnref{eq: taylor} holds is non-empty implying the statement of the lemma.
\end{proof}

Now, we state a useful result on the generalization of Parallelogram law, which provides a characterization for showing a Banch space is isomorphic to a Hilbert space. Let $\sum_{\epsilon(n)}$ denote all possible sequences $(\epsilon_1,\ldots,\epsilon_n)$ of $\curlybracket{\pm 1}$'s.
\begin{proposition}[Proposition 3.1,~\cite{Kwapień1972}]\label{prop: kwapien}
A real or complex Banach space $\cB$ is isomorphic to a Hilbert space if and only if
there exists a constant $C>0$ such that
\[
    C^{-1} \sum_{i=1}^{n} \| f_i \|_{\cB}^2 \leq \frac{1}{2^n} \sum_{\epsilon(n)} \left( \left\| \sum_{i=1}^{n} \epsilon_i f_i \right\|^2 \right) \leq C \sum_{i=1}^{n} \| f_i \|_{\cB}^2,
    \]
for any positive integer $n > 1$ and any $f_1, f_2, \ldots, f_n$ in $\cB$.

\end{proposition}

With this we proof the claim of \lemref{lemma: existence} in the following:

\begin{proof}[of \lemref{lemma: existence}]
    First, assume that $\cB$ is separable. The proof of the non-separable case extends from the separable case. 
    
    Assume for the sake of contradiction that there exists $\msf F: \cB \to \reals$ that is both $\mu$-strongly convex and $\gamma$-smooth for some $\mu > 0$ and $\gamma < \infty$. Using \lemref{lemma: taylor}, there exists $f$ and $g^* \in \partial_{f} \msf F$ such that for all $f' \in \cB$,
    \begin{equation}
\msf F(f + \lambda f') = \msf F(f) + \lambda \langle f',g^* \rangle + \frac{\lambda^2}{2} \langle f',\msf M(f') \rangle + o(\lambda^2) \quad (\lambda \to 0) \label{eq1}
\end{equation}
for a symmetric operator $\msf M$. Since $\msf F$ is $\gamma$-smooth, we have
\begin{align}
    \msf F(f' + \lambda f') \le \msf F(f') + \langle \lambda f', g^* \rangle + \frac{\gamma \lambda^2}{2}||f'||_\cB^2 \label{eq2}
\end{align}
Using $\mu$-strong convexity of $\msf F$ we get
\begin{align}
    \msf F(f' + \lambda f') \ge \msf F(f') + \langle \lambda f', g^* \rangle + \frac{\mu \lambda^2}{2}||f'||_\cB^2 \label{eq3}
\end{align}
Combining \eqnref{eq1}-\eqref{eq3}, we get
\begin{align}
   C_1||f'||_\cB^2  \le \langle f', \msf M(f') \rangle_{\cB} \le C_2||f'||_\cB^2 \label{eq: bound}
\end{align}
for some $0 < C_1 \le C_2$. Now, consider the following inner product $\langle\cdot,\cdot\rangle_{H}: \cB \times \cB \to \reals$ on the space $\cB$
\begin{align*}
    \forall h_1,h_2 \in \cB,\quad \langle h_1,h_2\rangle_{H} = \langle h_1, \msf M(h_2) \rangle_\cB
\end{align*}
Symmetry of $\langle\cdot,\cdot\rangle_{H}$ follows by the symmetry of the operator $\msf M$.
Using \eqnref{eq: bound} we have 
\begin{align*}
    C_1||f'||_\cB^2  \le \langle f', f' \rangle_{H} \le C_2||f'||_\cB^2 
\end{align*}
But alternately we can also write
\begin{align}
   \frac{1}{C_2} \langle f', f' \rangle_{H} \le ||f'||_\cB^2 \le  \frac{1}{C_1} \langle f', f' \rangle_{H}
\end{align}
Now, we show the condition for \propref{prop: kwapien} is true for the norm $||\cdot||_\cB$, that yields the isomorphism of $(\cB, ||\cdot||_{\cB})$ to $(\cB, ||\cdot||_{H})$.

Note that for any $n$ choice of functions $f_1,f_2,\ldots,f_n \in \cB$
\begin{align*}
 \sum_{\epsilon(n)} \left( \left\| \sum_{i=1}^{n} \epsilon_i f_i \right\|_H^2 \right)   = 2^{n} \sum_{i} \left\| f_i\right\|^2_H 
\end{align*}
But we have
\begin{align*}
 &C_2 \sum_{\epsilon(n)} \left( \left\| \sum_{i=1}^{n} \epsilon_i f_i \right\|_\cB^2 \right) \ge \sum_{\epsilon(n)} \left( \left\| \sum_{i=1}^{n} \epsilon_i f_i \right\|_H^2 \right)\\
 \implies  &C_2 \sum_{\epsilon(n)} \left( \left\| \sum_{i=1}^{n} \epsilon_i f_i \right\|_\cB^2 \right) \ge C_1 2^{n} \sum_{i} \left\| f_i\right\|^2_\cB  
\end{align*}
Similarly,
\begin{align*}
 &C_1 \sum_{\epsilon(n)} \left( \left\| \sum_{i=1}^{n} \epsilon_i f_i \right\|_\cB^2 \right) \le \sum_{\epsilon(n)} \left( \left\| \sum_{i=1}^{n} \epsilon_i f_i \right\|_H^2 \right)\\
 \implies  &C_1 \sum_{\epsilon(n)} \left( \left\| \sum_{i=1}^{n} \epsilon_i f_i \right\|_\cB^2 \right) \le C_2 2^{n} \sum_{i} \left\| f_i\right\|^2_\cB  
\end{align*}
Combining the equations, we have 
\begin{align*}
    C^{-1} \sum_{i} \left\| f_i\right\|^2_\cB   \le \frac{1}{2^n} \sum_{\epsilon(n)} \left( \left\| \sum_{i=1}^{n} \epsilon_i f_i \right\|_\cB^2 \right) \le C \sum_{i} \left\| f_i\right\|^2_\cB 
\end{align*}
where $C = \frac{C_2}{C_1}$. Thus, we have shown that $(\cB, ||\cdot||_\cB)$ is isomorphic to $(\cB, ||\cdot||_H)$ using \propref{prop: kwapien}.

Now, consider the non-separable case. Using the proof above every separable (open or closed) subspace $\cS \subset$ $\cB$ is isomorphic to a Hilbert space, which is sufficient to show that the space $(\cB, ||\cdot||_\cB)$ is isomorphic to a Hilbert space using \propref{prop: kwapien}.

For a separable subspace \(\cS \subset \cB\), denote by \(C_{\cS}\) the infimum of the constants that satisfy \propref{prop: kwapien}. Now, define:
\[
    C_{\sf{sup}} := \sup \left\{ C_{\cS} \,\middle|\, \cS \subset \cB, \text{ closed, separable} \right\}
\]

Note that if \(C_{\sf{sup}} < +\infty\), then it clearly satisfies the inequalities above for the whole space \(\cB\). Assume on the contrary that \(C_{\sf{sup}} = +\infty\), so for all \(n \in \mathbb{N}\) there exists a separable subspace \(\cS_n\) such that \(C_{\cS_n} \geq n\). 
Now, define by $\cS$ the closure of the span of countably many subsets $\cS_n$, i.e., $\cS = \overline{\operatorname{span} \left( \bigcup_{n} \cS_n \right)}.$

Since closure of countable union separable space is separable we note that \(\cS\) is separable and \(\cS_n \subset \cS\) for all \(n\). Thus, \(C_{\cS_n} \leq C_{\cS}\), so \(C_{\cS} \geq n\) for all \(n \in \mathbb{N}\). But this gives a contradiction because \(C_{\cS}\) is finite due to the construction. Thus, \(C_{\sf{sup}} < +\infty\), and hence the conditions for \propref{prop: kwapien} are satisfied, which implies that \((\cB, \|\cdot\|_{\cB})\) is isomorphic to a Hilbert space.
\end{proof}

\subsection{Proof of \lemref{lemma: natdirection}}\label{app: natdirection}

\begin{proof}
    First, note that for any \( f' \in \cC \), the map \( h \mapsto \mf D_{\Phi}(h, f') \) is convex. Secondly, note that for any \( f, f' \) we can write
    \begin{align}
      \partial_f \mf D_{\Phi}(\cdot, f') \in \partial_f \Phi - \partial_{f'} \Phi \label{eq: proj}
    \end{align}

    By definition, \(\Pi_{\cB_0}^{\Phi}(f')\) is the minimizer of \(\mf D_{\Phi}(f, f')\) in the first component over \(\cB_0 \cap \cC\). This means that for any \( g \in \cB_0 \cap \cC \),
    \begin{align*}
        \mf D_{\Phi}(\Pi_{\cB_0}^{\Phi}(f'), f') \le \mf D_{\Phi}(g, f')
    \end{align*}
    
    Using the first-order condition of convexity, we get
    \begin{align*}
        \inner{g - \Pi_{\cB_0}^{\Phi}(f'), \partial_{\Pi_{\cB_0}^{\Phi}(f')} \mf D_{\Phi}(\cdot, f')}_\cB \ge 0, \quad \forall\, g \in \cB_0 \cap \cC
    \end{align*}
    
    Substituting \eqref{eq: proj} into the above inequality, we have
    \begin{align*}
        \inner{g - \Pi_{\cB_0}^{\Phi}(f'), \partial_{\Pi_{\cB_0}^{\Phi}(f')} \Phi - \partial_{f'} \Phi}_\cB \ge 0
    \end{align*}
    
    Setting \( g = f \) in the above inequality, we get
    \begin{align*}
        \inner{f - \Pi_{\cB_0}^{\Phi}(f'), \partial_{\Pi_{\cB_0}^{\Phi}(f')} \Phi - \partial_{f'} \Phi}_\cB \ge 0
    \end{align*}
    
    This completes the proof.
\end{proof}

\newpage


\vskip 0.2in
\bibliography{ref}

\end{document}